\documentclass{article}

\usepackage[preprint]{neurips_2026}
 
\usepackage[utf8]{inputenc}
\usepackage[T1]{fontenc}
\usepackage{hyperref}
\usepackage{url}
\usepackage{booktabs}
\usepackage{amsfonts}
\usepackage{amssymb}
\usepackage{amsmath}
\usepackage{nicefrac}
\usepackage{microtype}
\usepackage{graphicx}
\usepackage{xcolor}
\usepackage{algorithm}
\usepackage{algorithmic}
\usepackage[capitalize]{cleveref}
\usepackage{mathtools}
\usepackage{bm}
\usepackage{amsthm}
\usepackage{multirow}
\usepackage{makecell}
\usepackage{array}
\usepackage{colortbl}
\usepackage[all]{nowidow}
\usepackage{rotating}
\usepackage{uoftcolors}
\usepackage{xspace}

\hypersetup{colorlinks, allcolors=uoftsecondaryblue}

\setcitestyle{numbers,square,comma}

\colorlet{bestblue}{uoftblue}
\colorlet{qampurple}{uoftsecondaryblue}
\colorlet{qamrowbg}{uoftsecondaryblue!10}
\definecolor{rangegray}{RGB}{140, 140, 140}

\crefname{appendix}{Appendix}{Appendices}
\Crefname{appendix}{Appendix}{Appendices}

\newcommand{\valarr}[4]{\makecell{{\fontsize{3.5}{4}\selectfont\textcolor{rangegray}{[#2]\hphantom{$\to$}[#4]}}\\\textcolor{rangegray}{#1}\,$\to$\,#3}}
\newcommand{\bestarr}[4]{\makecell{{\fontsize{3.5}{4}\selectfont\textcolor{rangegray}{[#2]\hphantom{$\to$}[#4]}}\\\textcolor{rangegray}{#1}\,$\to$\,\textcolor{bestblue}{\textbf{#3}}}}
\newcommand{\offbestarr}[4]{\makecell{{\fontsize{3.5}{4}\selectfont\textcolor{rangegray}{[#2]\hphantom{$\to$}[#4]}}\\\textcolor{bestblue}{\textbf{#1}}\,$\to$\,#3}}
\newcommand{\bothbestarr}[4]{\makecell{{\fontsize{3.5}{4}\selectfont\textcolor{rangegray}{[#2]\hphantom{$\to$}[#4]}}\\\textcolor{bestblue}{\textbf{#1}}\,$\to$\,\textcolor{bestblue}{\textbf{#3}}}}

\newcommand{\qambestarr}[4]{\makecell{{\fontsize{3.5}{4}\selectfont\textcolor{rangegray}{[#2]\hphantom{$\to$}[#4]}}\\\textcolor{rangegray}{#1}\,$\to$\,\textcolor{qampurple}{\textbf{#3}}}}
\newcommand{\qamoffbestarr}[4]{\makecell{{\fontsize{3.5}{4}\selectfont\textcolor{rangegray}{[#2]\hphantom{$\to$}[#4]}}\\\textcolor{qampurple}{\textbf{#1}}\,$\to$\,#3}}
\newcommand{\qambothbestarr}[4]{\makecell{{\fontsize{3.5}{4}\selectfont\textcolor{rangegray}{[#2]\hphantom{$\to$}[#4]}}\\\textcolor{qampurple}{\textbf{#1}}\,$\to$\,\textcolor{qampurple}{\textbf{#3}}}}

\newtheorem{proposition}{Proposition}

\newcommand{\E}{\mathbb{E}}
\newcommand{\R}{\mathbb{R}}
\newcommand{\cS}{\mathcal{S}}
\newcommand{\cA}{\mathcal{A}}
\newcommand{\cD}{\mathcal{D}}

\newcommand{\qpilots}{\texttt{QPILOTS}\xspace}
\newcommand{\qpilotu}{\texttt{QPILOTS-U}\xspace}
\newcommand{\qpilotm}{\texttt{QPILOTS-M}\xspace}

\title{\texttt{QPILOTS}:\\Efficient Test-Time Q-Steering for Flow Policies}

\author{%
  Yifan Ruan$^{1,2}$\thanks{Corresponding author: \texttt{yifan@cs.toronto.edu}.} \quad
  Chenyang Cao$^{1}$ \quad
  Andreas Burger$^{1,2}$ \\ 
  \bf Ali Pesaranghader$^{3}$ \quad
  Kaveh Kamali$^{3}$ \quad
  Jaehong Kim$^{3}$ \quad
  Nandita Vijaykumar$^{1,2}$ \\
  \bf Alan Aspuru-Guzik$^{1,2}$ \quad
  Igor Gilitschenski$^{1,2}$ \quad
  Nicholas Rhinehart$^{1}$ \\[4pt]
  {\normalfont $^{1}$University of Toronto \quad $^{2}$Vector Institute \quad $^{3}$LG Electronics}
}

\begin{document}

\maketitle

\begin{abstract}
Flow-matching and diffusion policies are expressive action generators, but optimizing them with temporal-difference reinforcement learning (RL) remains difficult. Effective policy extraction requires exploiting the critic's action gradient, yet directly backpropagating this signal through a multi-step denoising process can be numerically unstable.
Existing methods work around this either by discarding gradient information, distilling the policy into a simpler one-step actor, or repeatedly fine-tuning the denoising policy as the critic improves.
We propose \qpilots, a method that leaves the original policy unmodified and steers the denoising process at inference time. At each denoising step, instead of evaluating the critic on the noisy intermediate action where critic predictions are unreliable, we first project that intermediate state to an estimate of the final clean action and compute the critic gradient there. We introduce two variants: \qpilotu uses a fast single-point approximation, while \qpilotm draws differentiable posterior samples via a learned auxiliary network. On a standard offline-to-online RL benchmark, \qpilots{} achieves the best aggregate performance, reaching an average success rate of $90\%$ across 50 tasks. We also apply \qpilots{} to steer a large, frozen, pretrained Vision-Language Action (VLA) foundation model, outperforming or matching prior inference-time approaches across six manipulation tasks in simulation.
\end{abstract}


\begin{figure}[h]
\centering
\includegraphics[width=3.5in]{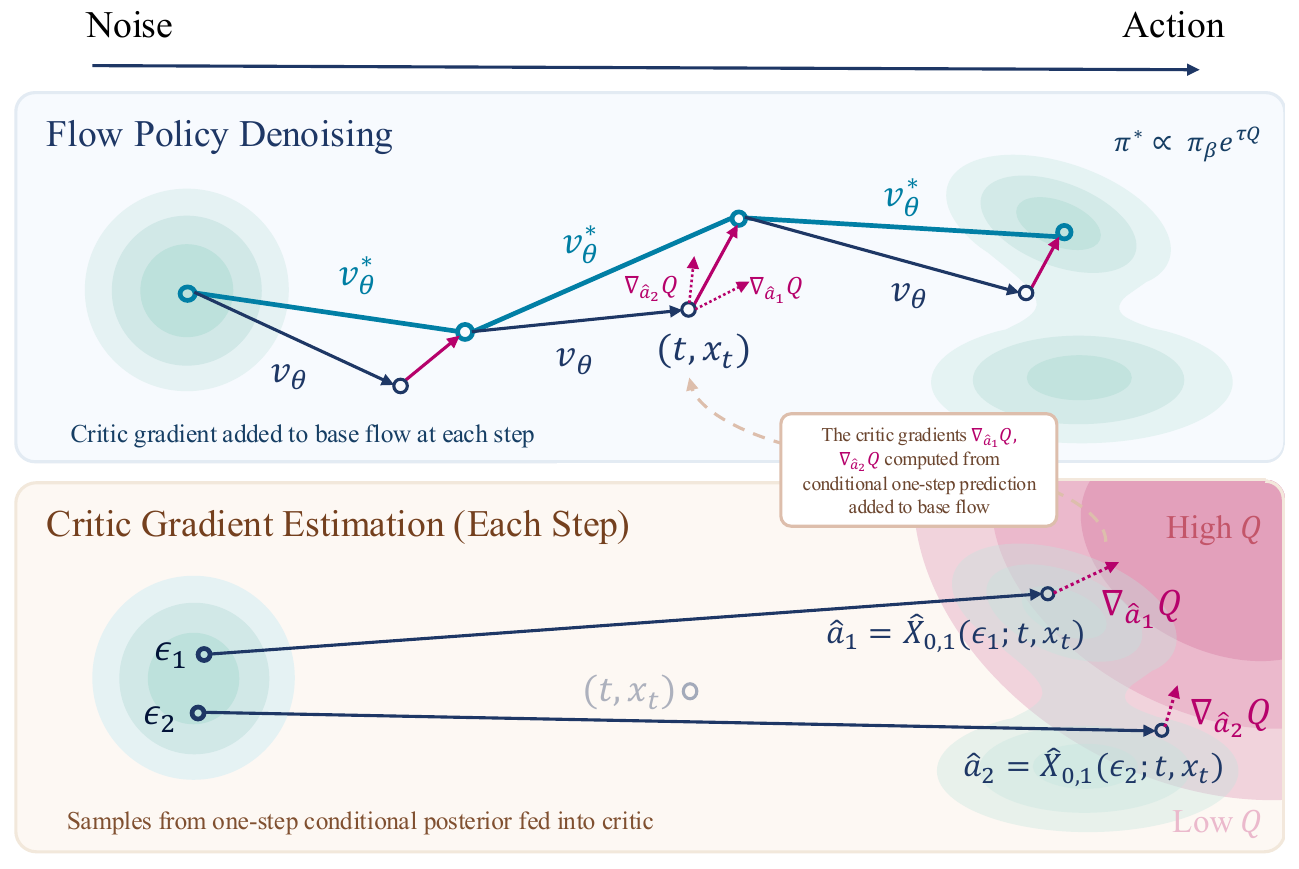}
\includegraphics[width=1.9in]{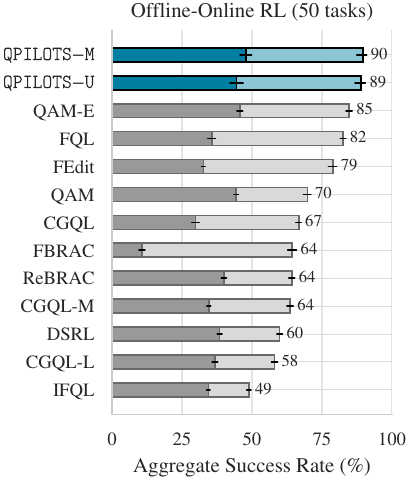}
\caption{\textbf{\texttt{QPILOTS} turns policy extraction into flow-time tilted sampling.}
\textbf{Left:} For a base policy $\pi_\beta$ and target critic $Q$, KL-regularized policy improvement yields the critic-tilted target
$\pi^*(a|s)\propto \pi_\beta(a|s)\exp(\tau Q(s,a))$.
The $Q$ function induces a log-tilt potential whose gradient gives a generation-time drift correction to the base policy's flow estimator $v_\theta$. \qpilots{} efficiently estimates this gradient to steer without updating the base policy.
\textbf{Right:} Aggregated offline-to-online success rates on 50 OGBench tasks \cite{park2025ogbench}.}
\label{fig:banner}
\end{figure}

\section{Introduction}
\label{sec:intro}
Flow-matching and diffusion policies have become the de facto standard way to represent complex, multi-modal action distributions in continuous control and vision-language-action models (VLAs)~\citep{lipman2023flow,black2024pi0,chi2025diffusionpolicy} and there is sustained interest in improving these policies with downstream rewards~\citep{park2025fql,wagenmaker2025dsrl,dong2025expo,li2026qam}. Formally, this can be framed as \textit{behavior-constrained} policy learning, where the agent is equipped with a base flow policy $\pi_\beta(a|s)$ trained on an offline dataset $\cD$ and improves it online with a learned critic, $Q$, without straying too far from the base policy. Recently, ~\citet{park2024bottleneck} has shown that the main challenge is to extract this optimal policy from $(\pi_\beta, Q)$ rather than learning $Q$ itself.
We refer to this mapping from a base policy and critic to an improved action distribution as \emph{policy extraction}. 
An extractor may either modify the policy during training or steer its samples at inference time.

For simple reparameterized policies, the critic's action gradient $\nabla_a Q(s,a)$ can guide policy updates. For flow policies, this same signal must pass through a multi-step generation, and backpropagation through the denoising process is often unstable or expensive~\citep{wang2023dql,ren2024diffusion,mcallister2025flow}. To circumvent this instability, prior work distills the original flow policy into a one-step approximation~\citep{park2025fql, li2026qam} or uses finetuning to target the optimal policy ~\citep{psenka2024qsm, fang2025dac, li2026qam}. In practice, the former sacrifices expressivity, and the latter requires tuning policy parameters, hindering its scalability to large-capacity models due to the prohibitive computational overhead. 
This motivates another class of methods to apply guidance during \emph{inference}, forgoing the need for finetuning. However, in offline-online Reinforcement Learning (RL) benchmarks, they have thus far been outperformed by methods that use finetuning~\citep{li2026qam}, and often have limited expressivity~\citep{dong2025expo}, high cost~\citep{hansenestruch2023idql}, or apply guidance at noisy actions where the critic is not calibrated~\citep{wagenmaker2025dsrl, li2026qam}. 

Addressing these limitations, we introduce \texttt{QPILOTS} (Q-guided Posterior Inference for Learning Off-policy and Test-time Steering). Our main method, \qpilotm, builds on Meta Flow Maps (MFM)~\citep{potaptchik2026mfm} to construct an asymptotically unbiased estimator of clean-action Q-value gradients during flow sampling. This allows a standard off-policy critic to steer the base policy without evaluating the critic on noisy intermediate latents. In turn, \qpilotm{} recovers much of the useful gradient information available to training-time extractors without the instability of backpropagating through the full denoising chain or the expressivity loss of one-step distillation.  

Empirically, we find that \qpilots{} achieves the best aggregate score among competitive training-time ~\citep{wang2023dql, park2025fql, li2026qam} and inference-time baselines~\citep{hansenestruch2023idql, wagenmaker2025dsrl, dong2025expo} on the OGBench~\citep{park2025ogbench} offline/offline-to-online RL benchmark, establishing it as a powerful solution to policy extraction. We further apply \qpilots to steer a state-of-the-art pretrained generalist $\pi_{0.5}$~\citep{black2025pi05} without modifying the base flow, demonstrating that \qpilots successfully learns from online interactions even with a behavior Q-value, surpassing or matching other inference-time approaches on $6$ tasks from the LIBERO-90 suite~\citep{liu2023libero}. Taken together, our contribution is an efficient inference-time steering recipe that applies RL to flow policies by leveraging clean-space Q-values to guide the intermediate generation process. 

\section{Related Work}
\label{sec:related}

\paragraph{Offline-to-Online Reinforcement Learning.}
\label{sec:related:rl}
Offline-to-online RL methods leverage offline datasets to pre-train policies and value functions, then fine-tune them through online interaction to improve sample efficiency~\citep{xie2021policy,lee2022offline,song2022hybrid,nakamoto2023calql,tarasov2023rebrac,zhou2024wsrl,li2025three}.
A common recipe uses the same offline RL objective during both phases, with the offline dataset loaded into a replay buffer alongside newly collected transitions~\citep{ball2023rlpd}.
While purely online methods can treat the offline data as additional off-policy experience, offline pre-training tends to produce stronger results on challenging sparse-reward tasks~\citep{li2025qc,li2026qam}. We adopt this training setup for the critic and focus on the complementary challenge of \textit{policy extraction}, that is selecting high-value actions with the policy while remaining close to the behavior support~\citep{fujimoto2021td3bc,kostrikov2022iql,nakamoto2023calql}. This emphasis follows the observation that policy extraction and test-time generalization can be the limiting factors in offline RL, even when the value function is well-trained~\citep{park2024bottleneck}. It also affords our approach inherent deployment flexibility, allowing the same inference-time extractor to be used offline, online, or during fine-tuning.

\paragraph{RL with Diffusion and Flow Policies.}
\label{sec:related:diffusion_flow_rl}
Diffusion and flow-matching policies have become a dominant policy class in both imitation learning and RL due to their ability to represent complex, multi-modal action distributions~\citep{janner2022planning,ajay2022conditional,kang2023efficient,he2023diffcps,lu2023contrastive,hansenestruch2023idql,ding2024consistency,chen2024diffusion,ren2024diffusion,chi2025diffusionpolicy}.
A central challenge is optimizing these policies against a critic $Q(s, a)$ when the generation process involves a multi-step denoising chain. 
Unlike prior work that categorizes approaches by how the value function is used~\citep{li2026qam}, we classify them according to how the base policy is treated during training and inference.
Under this lens, existing approaches can be categorized into three strategies.

\textit{1. Back-Propagating Through Time (BPTT)} offers a straightforward way to use value signals by differentiating the critic through the full denoising chain to maximize Q~\citep{wang2023dql,he2023diffcps,zhang2024entropy}. 
Related policy-gradient formulations instead make likelihoods or advantage ratios tractable for flow policies and updating the policy itself~\citep{mcallister2025flow,zhang2025reinflow}. While conceptually simple, computing gradients through the long unrolled chain of many-step flows is highly prone to numerical instability~\citep{park2025fql}.

\textit{2. Policy Fine-Tuning without BPTT} avoids full-chain gradients by changing the policy objective. Some methods distill the multi-step flow into a simplified one-step policy, improving stability but giving up part of the original model's iterative expressivity~\citep{ding2024consistency,park2025fql,espinosadice2025shortcut, wang2026meanflowql}.
Other methods construct a step-wise training objective that uses critic information at intermediate denoising states. Methods like QAM~\citep{li2026qam}, alongside other gradient-approximated or value-weighted variants~\citep{psenka2024qsm,fang2025dac,frans2025cfgrl}, construct specialized losses where the critic gradient only passes through the base policy at isolated denoising steps.
While these avoid the instability of backprop, they couple the base flow and the critic through a specialized loss and require re-training the policy whenever the critic changes.

\textit{3. Inference-time steering} modifies the action-generation process at test time without impacting training.
DSRL~\citep{singh2021parrot,wagenmaker2025dsrl} trains a noise-space RL policy that perturbs the input latents of a frozen flow.
Residual and edit policies~\citep{johannink2019residual,ankile2025imitation,dong2025expo,yuan2025decorator} train a separate corrector that acts on completed actions.
Best-of-$N$~\citep{hansenestruch2023idql,nakamoto2024vgps,mark2025policyagnostic,attarian2026ufops} samples $N$ candidates from the base policy and picks the highest-Q one, scaling poorly with action dimension.
All of these consult the critic either after a full action has been produced or via a separately trained auxiliary policy.
Our method instead injects a noiseless-action Q-gradient into the Euler integrator at \emph{every} denoising step, guiding the trajectory mid-integration without training an edit policy or modifying the base flow weights.

\paragraph{Flow Steering and Guided Generation.}
\label{sec:related:flow_steering}
Classifier guidance~\citep{dhariwal2021classifier,ho2022classifier} steers generation by adding reward gradients to the diffusion drift. Evaluating the reward at noisy states, however, yields a biased estimator of the tilted target $p_\theta \propto p_\beta \, e^{r}$, since rewards are typically calibrated only on terminal data~\citep{bansal2024universal}. Universal Guidance~\citep{bansal2024universal, attarian2026ufops} mitigates this by evaluating the reward at the Tweedie denoised estimate $\hat x_1$, but incurs posterior-mean bias when the conditional $p_{1|t}(\cdot \mid x_t)$ is multimodal. Thus existing guidance directions differ in whether they use the noisy-space gradient $\nabla Q(s,x_t)$, a denoised-point gradient $\nabla Q(s,\hat x_1)$, or the true posterior tilt gradient $\nabla_{x_t}\log \E_{x_1\sim p_{1|t}} e^{\tau Q(s,x_1)}$. Particle methods such as SMC samplers~\citep{wu2023tds}, DynaGuide~\citep{du2025dynaguide} reduce related guidance bias through resampling, noised-sample classifiers, or verification, but add particles, auxiliary training, or sample-selection cost. Meta Flow Maps~\citep{potaptchik2026mfm} and related stochastic flow-map models~\citep{holderrieth2026diamond} address the challenge of posterior-sampling by producing differentiable samples from $p_{1|t}$, enabling Monte Carlo estimators of the log-tilt-potential gradient. We adapt this line of work to RL by using a learned Q-function as the tilt scalar and by applying the resulting noiseless-action posterior gradient during policy sampling.

\section{Preliminaries and Background.}
\label{sec:prelim}

\paragraph{Reinforcement Learning.}
\label{sec:prelim:rl}

We consider a Markov Decision Process $\mathcal{M} = (\cS, \cA, P, \gamma, R, \mu)$, where $\cS$ is the state space, $\cA = \R^A$ is a continuous action space, $\Delta \mathcal{X}$ denotes the set of probability distributions over a set $\mathcal{X}$, $P : \cS \times \cA \to \Delta \cS$ is the transition function, $\gamma \in [0,1)$ is the discount factor, $R : \cS \times \cA \to \R$ is the reward function, and $\mu \in \Delta \cS$ is the initial state distribution. 
We assume that we have access to an offline dataset $\cD = \{(s_i, a_i, s_i', r_i)\}_{i=1}^{|\cD|}$ of transitions collected by a behavior policy $\pi_\beta$.
The goal of offline RL is to learn a policy $\pi_\theta : \cS \to \Delta \cA$ from $\cD$ that maximizes the expected discounted return
\begin{equation}
  \E_{s_0 \sim \mu,\, a_k \sim \pi(\cdot | s_k),\, s_{k+1} \sim P(\cdot | s_k, a_k)} \left[ \sum_{k=0}^{\infty} \gamma^k R(s_k, a_k) \right].
  \label{eq:return}
\end{equation}

\paragraph{Flow Matching.}
\label{sec:prelim:flow}

A flow-matching generative model~\citep{lipman2023flow,liu2023rectifiedflow} learns a time-dependent velocity field $v_\theta : \R^d \times [0,1] \to \R^d$ that transports a, typically Gaussian, noise distribution $p_0 = \mathcal{N}(0, I)$ to a data distribution $p_1$ via an ordinary differential equation (ODE)
\begin{equation}
  \dot{x}_t = v_\theta(x_t, t), \qquad x_0 \sim p_0.
  \label{eq:flow_ode}
\end{equation}
Under the linear interpolant $x_t = (1-t)\,x_0 + t\,x_1$, the velocity field is trained with
\begin{equation}
  \mathcal{L}_{\mathrm{FM}}(\theta) = \E_{t \sim \mathcal{U}[0,1],\, x_0 \sim p_0,\, x_1 \sim p_1} \left[ \| v_\theta(x_t, t) - (x_1 - x_0) \|^2 \right],
  \label{eq:fm_loss}
\end{equation}
known as the \emph{conditional flow matching objective}.
At inference, samples from $p_1$ are generated by sampling noise from the base distribution $x_0\sim p_0$ and integrating the ODE via $K$ Euler steps $x_{i+1} = x_i + h \cdot v_\theta(x_i, t_i)$, where $h = 1/K$, $t_i = i h$.
To represent a state-conditioned policy $\pi_\theta(\cdot \mid s)$ in the RL setting, we condition the velocity field on the current state $s$, writing $v_\theta(x_t, t, s)$, and treat the integrated endpoint $x_1$ as the agent's action.

\paragraph{Stochastic Formulation.}
The deterministic ODE maps each initial noise $x_0$ to a unique endpoint $x_1$, leaving no room to ``branch'' toward higher-reward outcomes once an intermediate state $x_t$ has been produced. Replacing the ODE with a matching stochastic differential equation (SDE) gives a family of paths through any intermediate state, characterized by a conditional distribution $p_{1|t}(x_1 \mid x_t)$ over endpoints reachable from $x_t$. 
For any non-negative diffusion schedule $\{\sigma_t\}_{t \in [0,1]}$ and Brownian motion $B_t$, the stochastic differential equation (SDE)
\begin{equation}
  dX_t = \left[ v_\theta(X_t, t) + \frac{\sigma_t^2}{2} \nabla \log p_t(X_t) \right] dt + \sigma_t \, dB_t, \qquad X_0 \sim p_0,
  \label{eq:flow_sde}
\end{equation}

produces the same marginal density $p_t$ as the ODE at every flow-time $t$~\citep{song2021scorebased}. 
We use the common choice  $\sigma_t^2 = 2(1-t)/t$, that is associated with the linear interpolant \citep{lipman2023flow}.

\paragraph{Exponentially Tilted Generative Sampling.}
\label{sec:prelim:tilting}
We next review a generic recipe for steering a generative model toward high-scoring samples. This view is common in reward-guided generation, classifier guidance, and control-as-inference~\citep{song2021scorebased,potaptchik2026mfm,mihatsch2002risk}.
Let $p_1$ denote the terminal distribution produced by a base generative model, and
$\psi : \R^d \to \R$ be a scalar utility assigned to a completed sample.
\emph{Exponentially tilted sampling} seeks to sample from
\begin{equation}
  p_{\psi}(x) \propto p_1(x)\, e^{\psi(x)}
  \label{eq:tilted_distribution}
\end{equation}
Larger $\psi(x)$ raises the probability of $x$ under the tilted target. In our setting, the completed $x_1$ is an \emph{action} and $\psi$ will be a critic-induced utility defined in \Cref{sec:method:framework}.
To realize this target with a stochastic process, define the \emph{log-tilt potential}
\begin{equation}
  V_t(x_t) := \log \E_{x_1 \sim p_{1|t}(x_1 \mid x_t)}\!\left[ e^{\psi(x_1)} \right],
  \label{eq:tilt_potential}
\end{equation}
which measures how much high-score terminal probability is reachable from the intermediate state $x_t$. (Despite the symbol, $V_t$ is \emph{not} the RL value function of \Cref{sec:prelim:rl}.) Adding a drift correction proportional to $\nabla V_t$ to the base SDE yields the \emph{steered probability-flow ODE}~\citep{potaptchik2026mfm}
\begin{equation}
  \dot{x}^\star_t = v_\theta(x^\star_t, t) + \tfrac{\sigma_t^2}{2}\, \nabla V_t(x^\star_t), \qquad x_0^\star \sim p_0,
  \label{eq:steered_ode}
\end{equation}
whose endpoint distribution at $t=1$ is exactly $p_{\psi}$. Our method will instantiate $v_\theta$ with a base flow policy's velocity network and $\psi$ with a learned action-value function, so that the steered ODE generates Q-tilted actions. The challenge is estimating $\nabla V_t$, since $p_{1|t}(x_1 \mid x_t)$ is generally not available in closed form. The next section describes the construction we use to overcome this.

\paragraph{Meta Flow Maps.}
\label{sec:prelim:mfm}
Drawing exactly from $p_{1|t}(\cdot \mid x_t)$ would normally require an inner ODE solve starting from $x_t$, which is prohibitively expensive inside the outer Euler loop. A Meta Flow Map (MFM)~\citep{potaptchik2026mfm} amortizes this posterior sampling problem over contexts $(t,x)$. For each context, it defines an auxiliary probability-flow ODE with drift $\bar{v}_u(\cdot,t,x)$ that transports exogenous noise $\epsilon \sim p_0$ to the endpoint posterior $p_{1|t}(\cdot \mid x)$.
Here $u, w \in [0,1]$ are auxiliary MFM times, distinct from the outer flow time $t$ and the RL state $s$. The MFM learns the solution operators of this family with a single amortized network $\hat{X}_{u,w}(\bar{x},t,x)$, parameterized in residual form
\begin{equation}
  \hat{X}_{u,w}(\bar{x}; t, x) = \bar{x} + (w-u)\,\hat{v}_{u,w}(\bar{x}; t, x),
  \label{eq:mfm_residual}
\end{equation}
where $\hat{v}_{u,w}$ predicts the average velocity of the auxiliary trajectory between times $u$ and $w$.
Setting $u=0$ and $w=1$ gives the one-step posterior sampler
\[
  \hat{X}_{0,1}(\epsilon,t,x) = \epsilon + \hat{v}_{0,1}(\epsilon;t,x),
  \qquad \epsilon \sim p_0.
\]
In the state-conditioned policy setting, we condition this sampler on the RL state and use $\hat{X}_{0,1}(\epsilon,t,x,s)$ as an approximate sample from $p_{1|t}(\cdot \mid x,s)$.

Training combines two losses. The \emph{diagonal loss} anchors the instantaneous velocity $\hat{v}_{u,u}$ to the conditional flow matching target:
\begin{equation}
  \mathcal{L}_{\mathrm{diag}} = \E_{t, u \sim \mathcal{U}[0,1]} \left[ \left\| \hat{v}_{u,u}(\bar{I}_u; t, I_t) - (\bar{I}_1 - \bar{I}_0) \right\|^2 \right],
  \label{eq:diag_loss}
\end{equation}
where $I_t = (1-t)I_0 + t I_1$ and $\bar{I}_u = (1-u)\bar{I}_0 + u \bar{I}_1$ are coupled interpolants sharing the data endpoint $I_1 = \bar{I}_1 \sim p_1$, with $I_0, \bar{I}_0 \sim p_0$ drawn independently.
The \emph{consistency loss} enforces the semigroup property
$\hat{X}_{u,w} = \hat{X}_{m,w} \circ \hat{X}_{u,m}$ for $u < m < w$:
\begin{equation}
  \mathcal{L}_{\mathrm{cons}}
  =
  \E_{t,\,x,\,u<m<w,\,\bar{x}}
  \left[
    \left\|
      \hat{X}_{u,w}(\bar{x}; t, x)
      -
      \hat{X}_{m,w}\!\left(
        \hat{X}_{u,m}(\bar{x}; t, x); t, x
      \right)
    \right\|^2
  \right].
  \label{eq:consistency_loss}
\end{equation}

\paragraph{Efficient Gradient Estimation.}
With a trained MFM, the log-tilt-potential gradient can be estimated via the
\emph{MFM-G estimator}~\citep{potaptchik2026mfm}:
\begin{equation}
  \widehat{\nabla_{x_t} V_t}(x_t)
  =
  \nabla_{x_t}
  \log
  \frac{1}{N}
  \sum_{n=1}^{N}
  \exp\!\left(
    \psi\!\left(
      \hat{X}_{0,1}(\epsilon_n;\, t,\, x_t)
    \right)
  \right),
  \qquad
  \epsilon_n \sim p_0,
  \label{eq:mfm_g}
\end{equation}
where $\hat{X}_{0,1}(\epsilon_n;\, t,\, x_t)$ is an MFM sample from the
conditional endpoint distribution $p_{1|t}(x_1 \mid x_t)$.
Differentiation through $\hat{X}_{0,1}$ with respect to $x_t$ is tractable
because $\hat X$ represents a single feedforward neural network evaluation.
With an exact posterior sampler, as $N \to \infty$ the average converges to the posterior expectation in \cref{eq:tilt_potential}, so the MFM-G estimator converges to
$\nabla_{x_t} V_t(x_t)$ under the assumptions of \citet{potaptchik2026mfm}.
Thus, MFM-G provides differentiable estimates of the steering direction without inner ODE integration.

\section{\texttt{QPILOTS}: Q-Guided Posterior Inference for Learning Off-Policy and Test-Time Steering}
\label{sec:method}
We now present our approach for steering a base flow policy with a learned critic. Rather than coupling policy and critic through specialized training-time losses, \qpilots{} treats the base flow $\pi_\beta$ and critic as independent components, each obtained via their own standard objectives, and combines them only at action-generation time. The base flow can be trained from scratch alongside the critic or taken off-the-shelf as a frozen pretrained policy; either way, we add a Q-gradient steering term to the flow's Euler integrator at inference. Because the base policy is never modified by the critic, any flow-matching policy can be plugged in unchanged, including pretrained policies learned on different data or tasks.
We present two instantiations of \qpilots{} that trade off bias and computational cost. \qpilotu{} is a Universal-Guidance variant that uses a Tweedie-denoised point estimate, while \qpilotm{} uses Meta Flow Maps to estimate the posterior-tilt gradient from differentiable posterior samples. \cref{alg:mfm-steer} summarizes the shared inference-time procedure underlying both variants.

\paragraph{Inference-Time Q-Steering.}
\label{sec:method:framework}
Following standard relative-entropy policy search~\citep{peters2010reps,peng2019awr,abdolmaleki2018mpo,park2025fql,li2026qam}, we cast critic-guided policy extraction as a KL-regularized improvement step. The goal is to maximize $\E_{a\sim\pi(\cdot|s)}[Q(s,a)]$ subject to $D_{\mathrm{KL}}(\pi(\cdot|s)\,\|\,\pi_\beta(\cdot|s)) \leq \delta(s)$. The closed-form solution is the exponentially tilted policy $\pi^\star(a|s) \propto \pi_\beta(a|s)\exp(\tau(s) Q(s,a))$ for an appropriate Lagrange multiplier $\tau(s) > 0$. To curb overestimation we use a pessimistic ensemble critic
\begin{equation}
  \bar{Q}(s, a) := \tfrac{1}{J}\textstyle\sum_{j=1}^{J} Q_{\phi_j}(s, a) - \rho\,\mathrm{std}_{j}\!\left(Q_{\phi_j}(s, a)\right),
  \label{eq:pessimistic_critic}
\end{equation}
with ensemble size $J$ and pessimism weight $\rho$ matching \citet{li2026qam}. Substituting $\bar Q$ for $Q$ above yields the objective $\pi^\star(a|s) \propto \pi_\beta(a|s)\exp(\tau\bar Q(s,a))$, which is exactly the tilted target of \cref{sec:prelim:tilting} with utility $\psi_s(a) := \tau\bar Q(s,a)$. Plugging the state-conditioned base velocity $v_\theta(\cdot, t, s)$ into \cref{eq:steered_ode} as the drift, the controlled probability-flow ODE has marginal $\pi^\star(\cdot|s)$ at $t=1$. The state-conditioned log-tilt potential
\begin{equation}
  V_t(x_t;s) := \log \E_{p_{1|t}(x_1 \mid x_t,s)} \exp\!\left(\psi_s(x_1)\right)
  \label{eq:q_steering_value}
\end{equation}
governs the steering term $\tfrac{\sigma_t^2}{2}\nabla_{x_t} V_t(x_t;s)$. The remaining task now is estimating this gradient.

\begin{algorithm}[t]
\caption{\qpilots: Inference-Time Q-Steering for Flow Policies}
\label{alg:mfm-steer}
\begin{algorithmic}[1]
\REQUIRE Observation $s$, base velocity field $v_\theta$, critic ensemble $\{Q_{\phi_j}\}_{j=1}^{J}$, log-tilt-potential gradient estimator $\textsc{GradV}$, steering coefficient $\alpha$, flow steps $K$
\STATE $h \leftarrow 1/K$;\quad $x_0 \sim \mathcal{N}(0, I)$
\FOR{$i = 0, \ldots, K-1$}
    \STATE $t \leftarrow i / K$;\quad $\hat{v} \leftarrow v_\theta(x_i, t, s)$ \COMMENT{Base velocity}
    \IF{$i > 0$}
        \STATE $\widehat{\nabla V}_t \leftarrow \textsc{GradV}(x_i, t, s)$ \COMMENT{\qpilotu: \cref{eq:ug_grad};\ \ \qpilotm: \cref{eq:q_steering_grad}}
        \STATE $g \leftarrow \frac{\|\hat{v}\|}{\|\widehat{\nabla V}_t\| + \varepsilon}\, \widehat{\nabla V}_t$;\quad $\hat{v} \leftarrow \hat{v} + \alpha \cdot g$ \COMMENT{Rescale gradient and compute steering direction}
    \ENDIF
    \STATE $x_{i+1} \leftarrow x_i + h \cdot \hat{v}$ \COMMENT{Steer intermediate action}
\ENDFOR
\STATE \textbf{return} $a =\operatorname{clip}(x_K,\, {-1},\, 1)$ \COMMENT{Final action}
\end{algorithmic}
\end{algorithm}

\paragraph{\qpilotu: Universal Guidance with the Tweedie Estimate.}
\label{sec:method:ug}
The simplest estimator replaces the intractable expectation in
\cref{eq:q_steering_value} with a point estimate at the posterior mean.
For the linear-interpolant flow, Tweedie's identity gives the minimum mean squared error (MMSE) denoiser~\citep{bansal2024universal}
\begin{equation}
\E[X_1 \mid X_t=x_t,s] \approx x_t + (1-t)\,v_\theta(x_t, t, s) =: \hat{x}_1
\end{equation}
Approximating $V_t(x_t;s) \approx \psi_s(\hat{x}_1)=\tau\bar Q(s,\hat{x}_1)$
yields our \emph{Universal Guidance} estimator:
\begin{equation}
  \widehat{\nabla V}_t^{\mathrm{UG}}(x_t) = \tau\, \nabla_{x_t}\, \bar{Q}\!\left(s,\, \mathrm{clip}\!\left[x_t + (1-t)\, v_\theta(x_t, t, s),\, -1,\, 1\right]\right),
  \label{eq:ug_grad}
\end{equation}
computed by a single additional forward-backward pass through the base velocity field and critic. The clip is applied straight-through during the backward pass.
This adapts the Tweedie-denoised classifier-guidance construction of \citet{bansal2024universal} to the flow-matching setting, using a learned Q-function as the guidance signal. \citet{attarian2026ufops} previously applied the same construction to diffusion BC policies, with a contrastive success classifier or a time-to-success estimator playing the role of the guidance signal, instead of an off-policy RL critic. While simple, the estimator introduces bias by (i) replacing the posterior log-expectation in \cref{eq:q_steering_value} with a point estimate at the posterior mean (ii) the denoised endpoint estimate is exact only when the learned velocity equals the true conditional mean velocity. In exchange, it adds zero training time beyond the base flow and the critic.

\paragraph{\qpilotm: Estimating $\nabla V_t$ with Meta Flow Maps.}
\label{sec:method:mfm}
To potentially improve on \qpilotu, we instantiate the MFM-G estimator in \cref{eq:mfm_g} with $\psi(\cdot) = \tau \bar{Q}(s, \cdot)$ and condition the MFM on the state $s$:
\begin{equation}
  \widehat{\nabla V}_t^{\mathrm{MFM}}(x_t) = \nabla_{x_t} \log \frac{1}{N} \sum_{n=1}^{N} \exp\!\left( \tau\, \bar{Q}\!\left(s,\, \hat{X}_{0,1}(\epsilon_n;\, t,\, x_t,\, s) \right) \right), \qquad \epsilon_n \stackrel{\mathrm{iid}}{\sim} \mathcal{N}(0, I).
  \label{eq:q_steering_grad}
\end{equation}
Compared with \qpilotu, this replaces the deterministic point estimate $\hat{x}_1$ with a Monte Carlo log-sum-exp over posterior samples.
If the MFM samples from the true endpoint posterior, the estimator is consistent for the posterior-tilt gradient as $N \to \infty$, at the cost of $N$ forward passes through the MFM and critic per Euler step, plus one backward pass.

\paragraph{Gradient Rescaling.} The theoretical drift $\tfrac{\sigma_t^2}{2}\nabla V_t$ uses $\sigma_t^2 = 2(1-t)/t$, which diverges as $t\to 0$ and vanishes as $t\to 1$. 
As is common in implementations with endpoint-singular schedules \citep{song2021maximum}, we do not apply steering at $t=0$, where $\nabla V_0$ is uninformative since $p_{1|0}$ equals the unconditional base-policy action distribution and therefore provides no state-dependent steering signal. For $t>0$ we follow \citet{potaptchik2026mfm} in rescaling the gradient to match the base-drift magnitude
\begin{equation}
  v_\theta^{\mathrm{steered}}(x_t, t, s) = v_\theta(x_t, t, s) + \alpha\,g_t,\qquad g_t = \frac{\|v_\theta(x_t, t, s)\|}{\|\widehat{\nabla V}_t(x_t)\| + \varepsilon}\,\widehat{\nabla V}_t(x_t),
  \label{eq:grad_rescale}
\end{equation}
with $\varepsilon > 0$ a small numerical-stability constant. $v_\theta^{\mathrm{steered}}$ is the practical, magnitude-rescaled surrogate for the theoretical optimal $v_\theta^\star = v_\theta + (\sigma_t^2/2)\nabla V_t$ (analyzed in \cref{sec:method:theory}); both target $\pi^\star(\cdot|s)$. In practice we fix $\tau = 1$.

\paragraph{Training Settings.}
\label{sec:method:training}
\cref{alg:mfm-steer} treats the base flow trainer and the critic learner as black boxes. We instantiate it in two regimes. \emph{(i) From scratch (OGBench):} $v_\theta$, the critic ensemble, and (for \qpilotm) the MFM are trained jointly on offline data and continue training on offline+online replay, using the standard conditional flow-matching loss (\cref{eq:fm_loss}), the MFM loss (\cref{eq:diag_loss} and \cref{eq:consistency_loss}), and TD learning against a pessimistic target ensemble with next-actions drawn from the steered policy~\citep{li2026qam}. \emph{(ii) Frozen base policy (LIBERO):} we keep the public $\pi_{0.5}$-LIBERO checkpoint~\citep{black2025pi05} frozen and learn only a task-specific critic online. \qpilotu queries $v_\theta$ directly from the frozen base flow. \qpilotm additionally trains an MFM on LIBERO-90 demonstration data via the GLASS reparameterization of \citet{potaptchik2026mfm}, with the frozen base flow as the conditional teacher. Unlike the original MFM setting, where the auxiliary network is trained on the same dataset that produced the base flow, we have access only to LIBERO-90 (a small subset of $\pi_{0.5}$'s pretraining mixture, which is not public); the consequences of this mismatch are discussed in \cref{sec:exp:libero}. Full procedures are in \cref{app:implementation}.

\paragraph{Theoretical Guarantees.}
\label{sec:method:theory}
\qpilotm inherits the convergence guarantees of \citet{potaptchik2026mfm} in the RL setting for the idealized sampler that uses exact posterior samples and the unrescaled $\sigma_t^2/2$ drift. Under standard regularity on $v_\theta$, $\bar Q$, the exact posterior sampler, and $\sigma_t$ (full assumptions and proof in Appendix~\ref{app:proofs}), the distribution $\hat{p}_1^s$ produced by the the idealized version of \cref{alg:mfm-steer} with $K$ Euler steps and $N$ MFM posterior samples per step satisfies
\begin{equation*}
  W_2\!\left(\hat{p}_1^s,\, \pi^\star(\cdot | s)\right) \leq C\!\left(\frac{1}{\sqrt{K}} + \frac{1}{\sqrt{N}}\right),
  \qquad
  \mathrm{KL}\!\left(\hat{p}_1^s \,\|\, \pi^\star(\cdot | s)\right) \leq C\!\left(\frac{1}{K} + \frac{1}{N}\right),
\end{equation*}
for some $C > 0$ independent of $K$ and $N$. The two error sources, Euler-Maruyama time discretization and Monte Carlo gradient estimation, are controlled independently. \qpilotu, which replaces the log-expectation in $V_t$ with the Tweedie posterior-mean estimate, is Bayes-optimal only when $\pi_\beta$ is Gaussian; we do not claim a formal guarantee for it.

\section{Experiments}
\label{sec:experiments}

\begin{table}[t]
\centering
\setlength{\tabcolsep}{2pt}
\renewcommand{\arraystretch}{1.6}
\scriptsize
\begin{tabular}{lccccccccccc}
\toprule
 & \texttt{al} & \texttt{ag} & \texttt{hm} & \texttt{hl} & \texttt{scene} & \texttt{p33} & \texttt{p44} & \texttt{c2} & \texttt{c3} & \texttt{c4} & \texttt{all} \\
 & {\tiny 5 tasks} & {\tiny 5 tasks} & {\tiny 5 tasks} & {\tiny 5 tasks} & {\tiny 5 tasks} & {\tiny 5 tasks} & {\tiny 5 tasks} & {\tiny 5 tasks} & {\tiny 5 tasks} & {\tiny 5 tasks} & {\tiny 50 tasks} \\
\multicolumn{12}{@{}l@{}}{\makebox[\linewidth][l]{\textsc{\textbf{Gaussian}}~\hrulefill}} \\
ReBRAC & \offbestarr{94}{94,95}{98}{97,98} & \offbestarr{57}{53,60}{75}{70,79} & \valarr{69}{65,74}{58}{52,64} & \valarr{17}{15,20}{23}{19,27} & \valarr{65}{57,73}{99}{98,99} & \bestarr{79}{72,85}{100}{100,100} & \valarr{0}{0,0}{0}{0,0} & \valarr{9}{8,10}{97}{96,98} & \valarr{1}{0,1}{14}{9,19} & \bestarr{9}{6,11}{80}{79,80} & \valarr{40}{39,41}{64}{63,65} \\
\multicolumn{12}{@{}l@{}}{\makebox[\linewidth][l]{\textsc{\textbf{Backprop}}~\hrulefill}} \\
FBRAC & \valarr{2}{1,4}{89}{82,95} & \valarr{0}{0,0}{81}{71,88} & \valarr{39}{37,41}{49}{45,53} & \valarr{0}{0,0}{4}{3,4} & \valarr{50}{41,60}{99}{97,100} & \valarr{0}{0,1}{95}{89,100} & \valarr{15}{11,20}{92}{85,98} & \valarr{0}{0,0}{79}{76,83} & \valarr{0}{0,1}{5}{1,9} & \valarr{0}{0,0}{51}{46,57} & \valarr{11}{10,12}{64}{63,66} \\
\multicolumn{12}{@{}l@{}}{\makebox[\linewidth][l]{\textsc{\textbf{Fine-tuning}}~\hrulefill}} \\
FQL & \bestarr{76}{72,79}{99}{99,99} & \valarr{0}{0,0}{93}{87,97} & \valarr{68}{63,73}{87}{83,91} & \valarr{9}{7,11}{16}{14,17} & \valarr{79}{72,85}{91}{87,96} & \bestarr{70}{59,80}{100}{100,100} & \bestarr{5}{3,8}{100}{100,100} & \bestarr{46}{43,49}{100}{100,100} & \valarr{3}{2,4}{60}{54,64} & \bestarr{2}{1,4}{80}{79,80} & \valarr{36}{34,37}{82}{82,83} \\
QAM & \valarr{81}{78,84}{98}{97,98} & \valarr{18}{14,22}{53}{52,55} & \valarr{67}{64,69}{83}{78,87} & \valarr{11}{9,14}{21}{17,24} & \bestarr{97}{95,98}{100}{100,100} & \bestarr{99}{99,100}{100}{100,100} & \valarr{0}{0,0}{3}{0,8} & \bestarr{64}{62,66}{100}{100,100} & \valarr{3}{3,4}{70}{63,76} & \valarr{3}{2,4}{71}{67,74} & \valarr{44}{44,45}{70}{68,71} \\
QAM-E & \valarr{83}{80,86}{96}{95,97} & \bestarr{1}{0,3}{95}{91,97} & \valarr{59}{54,63}{83}{75,90} & \valarr{2}{1,3}{16}{10,23} & \bestarr{97}{96,98}{100}{100,100} & \bothbestarr{100}{100,100}{100}{100,100} & \offbestarr{39}{31,48}{99}{97,100} & \bestarr{65}{63,68}{100}{100,100} & \bestarr{5}{4,6}{79}{79,80} & \valarr{6}{4,9}{79}{79,80} & \valarr{46}{45,47}{85}{84,86} \\
\multicolumn{12}{@{}l@{}}{\makebox[\linewidth][l]{\textsc{\textbf{Steering}}~\hrulefill}} \\
DSRL & \valarr{61}{56,66}{90}{88,91} & \valarr{3}{1,4}{40}{38,43} & \valarr{53}{48,57}{77}{70,82} & \valarr{3}{2,5}{29}{20,39} & \bothbestarr{99}{99,100}{100}{100,100} & \valarr{87}{78,95}{87}{77,95} & \valarr{0}{0,0}{0}{0,0} & \valarr{74}{72,76}{99}{98,99} & \valarr{1}{1,2}{0}{0,0} & \valarr{2}{2,3}{78}{77,78} & \valarr{38}{38,39}{60}{59,61} \\
IFQL & \valarr{36}{32,39}{76}{72,80} & \valarr{1}{0,2}{17}{16,18} & \offbestarr{86}{85,87}{51}{45,58} & \offbestarr{24}{21,27}{7}{6,9} & \valarr{84}{80,88}{96}{95,98} & \bothbestarr{100}{100,100}{100}{99,100} & \valarr{0}{0,0}{0}{0,0} & \valarr{11}{10,12}{79}{79,80} & \valarr{0}{0,0}{0}{0,0} & \valarr{2}{1,3}{61}{59,63} & \valarr{34}{34,35}{49}{48,50} \\
FEdit & \valarr{58}{53,62}{96}{95,96} & \valarr{2}{1,3}{86}{83,90} & \valarr{22}{20,23}{80}{69,90} & \valarr{3}{2,3}{8}{5,11} & \valarr{62}{54,70}{99}{98,99} & \bestarr{99}{98,100}{100}{100,100} & \bestarr{34}{26,41}{100}{100,100} & \bestarr{40}{37,43}{100}{100,100} & \valarr{2}{2,3}{40}{33,46} & \bestarr{5}{3,7}{80}{79,80} & \valarr{33}{32,33}{79}{77,80} \\
CGQL & \valarr{76}{73,80}{94}{94,95} & \valarr{0}{0,2}{85}{82,88} & \valarr{60}{57,62}{71}{69,73} & \valarr{5}{4,5}{5}{4,6} & \valarr{38}{28,49}{85}{80,89} & \bestarr{48}{40,55}{100}{100,100} & \valarr{24}{16,34}{95}{88,100} & \valarr{38}{36,42}{97}{95,98} & \offbestarr{8}{7,9}{20}{20,20} & \valarr{0}{0,0}{14}{10,17} & \valarr{30}{28,31}{67}{66,68} \\
CGQL-M & \valarr{71}{68,73}{82}{81,84} & \valarr{4}{1,8}{57}{55,60} & \valarr{42}{40,43}{92}{86,97} & \valarr{6}{3,8}{34}{27,40} & \valarr{74}{65,84}{92}{89,95} & \offbestarr{100}{100,100}{97}{96,99} & \valarr{0}{0,0}{29}{19,39} & \valarr{41}{39,43}{98}{97,98} & \offbestarr{8}{7,9}{20}{20,20} & \valarr{1}{0,1}{35}{28,41} & \valarr{35}{34,35}{64}{62,65} \\
CGQL-L & \valarr{65}{62,67}{79}{77,80} & \valarr{3}{1,6}{57}{55,59} & \valarr{62}{57,67}{83}{81,85} & \valarr{6}{5,8}{8}{7,10} & \valarr{88}{84,92}{67}{58,76} & \bestarr{90}{86,93}{100}{100,100} & \valarr{0}{0,0}{25}{16,34} & \valarr{45}{43,47}{99}{99,100} & \offbestarr{8}{7,9}{20}{20,20} & \valarr{0}{0,1}{42}{38,46} & \valarr{37}{36,38}{58}{57,59} \\
\multicolumn{12}{@{}l@{}}{\makebox[\linewidth][l]{\textsc{\textbf{Ours}}~\hrulefill}} \\
\rowcolor{qamrowbg}[2pt][5pt]
\qpilotu & \valarr{82}{76,88}{96}{94,98} & \valarr{1}{0,4}{89}{85,94} & \qambestarr{63}{54,75}{99}{98,100} & \valarr{8}{2,14}{59}{40,72} & \valarr{89}{80,95}{98}{96,100} & \qambothbestarr{100}{100,100}{100}{100,100} & \qambestarr{16}{6,25}{100}{100,100} & \qambothbestarr{75}{71,80}{100}{99,100} & \valarr{7}{4,10}{69}{57,74} & \valarr{3}{0,6}{78}{76,80} & \valarr{44}{42,47}{89}{87,90} \\
\rowcolor{qamrowbg}[2pt][5pt]
\qpilotm & \valarr{82}{74,88}{97}{93,98} & \valarr{10}{0,14}{90}{85,95} & \qambestarr{71}{64,75}{99}{98,100} & \qambestarr{12}{6,16}{60}{38,70} & \valarr{95}{91,98}{99}{97,100} & \qambothbestarr{100}{99,100}{100}{100,100} & \qambestarr{17}{2,27}{100}{100,100} & \qambothbestarr{75}{70,81}{100}{99,100} & \valarr{7}{5,9}{72}{67,75} & \qamoffbestarr{10}{3,20}{79}{78,80} & \qambothbestarr{48}{46,50}{90}{87,91} \\
\bottomrule
\end{tabular}
\caption{\textbf{OGBench: Offline $\to$ Online Success Rate.} Each cell shows end-of-offline (gray) and end-of-online (colored) success rates with seed-level $95\%$ bootstrap CIs in brackets. Bold marks the best offline and best online result per column.\protect\footnotemark}
\label{tab:offline_to_online_arrow}
\end{table}
\footnotetext{All baseline numbers are reported directly from~\citet{li2026qam}; only our methods (\qpilotu, \qpilotm) were run from scratch.}

\subsection{OGBench: Offline and Offline-to-Online (All Components Trained from Scratch)}
\label{sec:exp:ogbench}

\paragraph{Setup.}
We evaluate on the OGBench~\citep{park2025ogbench} single-task benchmark, which spans 10 continuous-control domains (5 navigation, 5 manipulation) with 5 tasks per domain, for a total of 50 tasks. Each method first trains for $10^6$ offline gradient steps on the OGBench dataset, then performs $5 \times 10^5$ online steps that interleave environment interaction (one gradient update per one environment interaction) with gradient updates over the offline data and an online replay buffer. Manipulation domains use action chunking, navigation domains use single-step actions, per-domain $(H, \gamma, \rho)$ follow \citet{li2026qam} (Appendix~\ref{app:implementation:hparams}).
Critic pessimism is set to $\rho = 0.5$ throughout, except on the humanoid-maze domains, where we follow~\citet{li2026qam} and use $\rho = 0$.
We report mean evaluation success rate at the end of training over $8$ seeds and $50$ evaluation episodes per seed, with $95\%$ bootstrap CIs across seeds shown in brackets.

\paragraph{Baseline.}
We group baselines by how they extract a policy from $(\pi_\beta, Q)$:
(i) Gaussian-policy baseline ReBRAC~\citep{tarasov2023rebrac}
(ii) the backprop-based flow extractor FBRAC~\citep{wang2023dql}, which differentiates through the denoising chain
(iii) fine-tuning extractors that update the base flow with a critic-aware objective, including the one-step distillation FQL~\citep{park2025fql}, and the adjoint-matching methods QAM and QAM-EDIT (QAM-E) of~\citet{li2026qam}
(iv) inference-time steering methods that wrap a frozen flow, including DSRL~\citep{wagenmaker2025dsrl} for latent-space policy steering, IFQL~\citep{hansenestruch2023idql} for implicit Q-learning policy extraction with Best-of-$N$ ($N=32$, following \citet{li2026qam}), FEdit~\citep{dong2025expo,yuan2025decorator} for post-hoc action refinement via learned Gaussian edit policies, CGQL~\citep{dhariwal2021classifier} for conditional generation via classifier guidance, and the CGQL-MSE/Linex (CGQL-M / CGQL-L) reward-shaping variants of CGQL following~\citet{li2026qam}.
Per-domain hyperparameters for all baselines follow~\citet{li2026qam}.
For our method, we evaluate both \qpilotu (Tweedie denoised estimate) and \qpilotm (MFM posterior). Both share the base flow, critic ensemble, and Adam optimizer of \citet{li2026qam}. \qpilotm{} additionally trains the auxiliary meta flow map. The steering coefficient $\alpha$ is tuned per domain following \citet{li2026qam}'s protocol, using two tuning tasks per domain (tasks 1 and 4 for navigation, tasks 2 and 4 for manipulation), 4 seeds per configuration. The selected value is reused across all five tasks of that domain and evaluated on 8 fresh seeds. The sweep grid and selected values are in Appendix~\ref{app:implementation:hparams}.

\paragraph{Results.}
Both variants lead in aggregate online score (\qpilotm at 90\%, \qpilotu at 89\%), and are especially performant on the \texttt{humanoid} tasks, improving by 25 absolute percentage points over the closest baseline on \texttt{humanoid-large}. \qpilotu is surprisingly effective online given its simplicity, beating much more elaborate methods such as QAM-E, though it slightly underperforms on certain tasks offline. \qpilotm closes this offline performance gap by replacing the point estimate with the MFM posterior (\cref{eq:q_steering_grad}). We hypothesize that accurate posterior estimates may matter more offline where the replay buffer's action coverage is limited. \qpilots significantly outperforms CGQL-Linex and DSRL, likely owing at least partially to our method's ability to use a critic directly rather than distilling it into an auxiliary critic that operates over noised actions. Per-task learning curves across all 50 tasks are reported in Appendix~\ref{app:implementation:hparams}.

\subsection{LIBERO: Online RL (Frozen Base Policy, Learned Critic Ensemble)}
\label{sec:exp:libero}
\paragraph{Setup.}
We test whether both \qpilots variants can steer a large frozen pretrained base policy, a regime where full fine-tuning is impractical because the base flow has billions of parameters. We use the official $\pi_{0.5}$-LIBERO checkpoint of \citet{black2025pi05}, pretrained on a large closed-source robotics mixture and then finetuned on LIBERO-40, as the frozen base flow. We evaluate on six tasks from the larger LIBERO-90 suite~\citep{liu2023libero} on which $\pi_{0.5}$'s zero-shot success rate is between 5 and 50 percent. Per-task descriptions are in Appendix~\ref{app:implementation:hparams:libero}. For \qpilotm, we distill the auxiliary meta flow map on the public LIBERO-90 demonstrations (\cref{sec:method:training}), keeping the $\pi_{0.5}$-LIBERO weights frozen. Each method runs $5\times10^5$ online environment steps with action chunking matching the base policy. We report results over $5$ seeds with $95\%$ seed-level bootstrap confidence intervals. We follow DSRL's~\citep{wagenmaker2025dsrl} simulation setup, but switch to SARSA targets in place of Bellman targets. Constructing $a' \sim \pi^\star(\cdot|s')$ for the off-policy target would require sampling from $\pi_{0.5}$ inside every critic update, which is prohibitive at $3$B parameters. SARSA reuses the buffer's stored $a'$ and avoids this cost. We also conduct a value learning ablation study (Appendix~\ref{app:value_ablation}) to show that SARSA outperforms other feasible approaches (e.g., IQL, Monte-Carlo, one-step distillation) in LIBERO.

\paragraph{Baselines and Hyperparameters.}
We compare both \qpilots variants against DSRL~\citep{wagenmaker2025dsrl} 
, Best-of-$N$ (BoN, $N=5$) action selection under the learned critic, EXPO~\citep{dong2025expo}, and CGQL-LinEx~\citep{park2025fql}. All baselines other than DSRL share the same SARSA critic for fairness. We tune one knob per method per task. For DSRL we sweep the action magnitude $b_W$ over $\{1, 2, 3\}$, matching \citet{wagenmaker2025dsrl}.
For \qpilots{} and CGQL-LinEx  we tune the steering coefficient $\alpha$ over $\{0.1, 0.2, 0.3\}$. Selected per-task values are listed in \cref{tab:libero_hparams}. We also include SAC~\citep{haarnoja2018sac} as a from-scratch baseline; it achieves negligible performance within the $5\times10^5$-step online budget.

\paragraph{Results.}
\Cref{fig:libero} shows per-task online learning curves at $5\times10^5$ steps. Per-task end-of-training detailed success rates are tabulated in Appendix~\ref{app:implementation:hparams}. \qpilotu improves over the frozen $\pi_{0.5}$ base on all six tasks and is the best method on five of six (tasks 26, 31, 38, 60, 64), with the gap over the strongest baseline DSRL ranging from $4\%$ on task 60 (within the seed-level bootstrap CIs) to $55\%$ on task 64. BoN, EXPO, and CGQL-LinEx, which either lack a directional gradient signal or steer in noise space, plateau well below the steering methods. \qpilotm tracks \qpilotu on most tasks and beats it on task 59, but does not improve over \qpilotu in aggregate on this benchmark. One likely reason is that MFM theory assumes the auxiliary network is distilled on the same data mixture that produced the base flow \citep{potaptchik2026mfm}. $\pi_{0.5}$'s pretraining mixture is closed-source and the public LIBERO-90 demonstrations only partially overlaps with $\pi_{0.5}$'s training data. 
The resulting distillation gap might explain why the simpler Tweedie estimate of \qpilotu is more performant here.

\begin{figure}[t]
\centering
\resizebox{0.99\linewidth}{!}{
\includegraphics{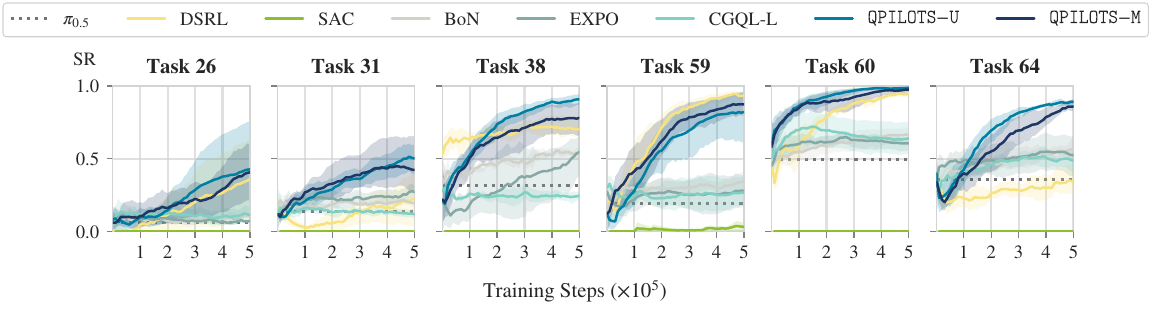}
}
\caption{
\textbf{LIBERO Online RL.} Per-task success rate over training for the LIBERO benchmark. Curves are time-weighted EMA-smoothed means over $5$ seeds; shaded bands show seed-level $95\%$ bootstrap CIs. Per-task descriptions, hyperparameters, and end-of-training numbers are in \cref{app:implementation:hparams:libero}.}
\label{fig:libero}
\end{figure}

\section{Discussion}
\label{sec:discussion}
We present \qpilots, an inference-time recipe for applying RL to flow policies that injects a clean-action Q-gradient into the Euler integrator at every denoising step, leaving the base flow and the critic to be trained with off-the-shelf objectives. Our main variant, \qpilotm, uses a Meta Flow Map to draw differentiable posterior samples, yielding a posterior-gradient estimator that inherits the asymptotic convergence guarantees of \citet{potaptchik2026mfm} under the true-posterior idealization. \qpilotu{} provides a lightweight alternative by replacing posterior sampling with a Tweedie point estimate. 
On OGBench offline-to-online, both variants outperform training-time extractors and prior inference-time steering methods, reaching $90\%$ aggregate success across $50$ tasks; on LIBERO with the frozen $\pi_{0.5}$ generalist, \qpilots improves over the zero-shot base on every evaluated task and surpasses DSRL on average. 

\paragraph{Limitations and Future Work.} Similar to previous work, the steering coefficient $\alpha$ requires manual tuning per domain (or per task on LIBERO). The gradient rescaling in \cref{eq:grad_rescale} avoids the singular $\sigma_t^2/2$ schedule but does not eliminate the need to pick a magnitude, and an adaptive choice tied to local critic curvature is a natural next step. We also see two promising extensions. One is to train a residual edit-flow analogue of \qpilotm in the spirit of QAM-E to further push performance. The second is to bring offline-trained value functions to large VLAs, where an offline-pretrained $Q$ would take advantage of our method's critic-agnostic nature, which extends the online setting in this work.

\newpage
\begin{ack}
This work was enabled in part by LG Electronics, Toronto AI Lab (Grant Ref No. 2025-2297), the Digital Research Alliance of Canada (\texttt{alliancecan.ca}), the NVIDIA Academic Grant Program, and Google TPU Research Cloud (TRC).
\end{ack}

\bibliographystyle{unsrtnat}
\bibliography{references}

@inproceedings{ball2023rlpd,
  author = {Philip J. Ball and Laura Smith and Ilya Kostrikov and Sergey Levine},
  title = {Efficient Online Reinforcement Learning with Offline Data},
  booktitle = {International Conference on Machine Learning (ICML)},
  pages = {1577--1594},
  year = {2023},
}

@article{mihatsch2002risk,
  title={Risk-sensitive reinforcement learning},
  author={Mihatsch, Oliver and Neuneier, Ralph},
  journal={Machine learning},
  volume={49},
  number={2},
  pages={267--290},
  year={2002},
  publisher={Springer}
}

@inproceedings{peters2010reps,
  title={Relative Entropy Policy Search},
  author={Peters, Jan and M{\"u}lling, Katharina and Alt{\"u}n, Yasemin},
  booktitle={AAAI Conference on Artificial Intelligence},
  year={2010}
}

@inproceedings{peng2019awr,
  title={Advantage-Weighted Regression: Simple and Scalable Off-Policy Reinforcement Learning},
  author={Peng, Xue Bin and Kumar, Aviral and Zhang, Grace and Levine, Sergey},
  booktitle={International Conference on Learning Representations},
  year={2020}
}

@inproceedings{abdolmaleki2018mpo,
  title={Maximum a Posteriori Policy Optimisation},
  author={Abdolmaleki, Abbas and Springenberg, Jost Tobias and Tassa, Yuval and Munos, R{\'e}mi and Heess, Nicolas and Riedmiller, Martin},
  booktitle={International Conference on Learning Representations},
  year={2018}
}

@inproceedings{tarasov2023rebrac,
  author = {Denis Tarasov and Vladislav Kurenkov and Alexander Nikulin and Sergey Kolesnikov},
  title = {Revisiting the Minimalist Approach to Offline Reinforcement Learning},
  booktitle = {Advances in Neural Information Processing Systems (NeurIPS)},
  year = {2023},
}

@inproceedings{park2025ogbench,
  author = {Seohong Park and Kevin Frans and Benjamin Eysenbach and Sergey Levine},
  title = {{OGBench}: Benchmarking Offline Goal-Conditioned {RL}},
  booktitle = {International Conference on Learning Representations (ICLR)},
  year = {2025},
}

@article{li2026qam,
  title={Q-learning with Adjoint Matching},
  author={Li, Qiyang and Levine, Sergey},
  journal={arXiv preprint arXiv:2601.14234},
  year={2026}
}

@inproceedings{park2025fql,
  author = {Seohong Park and Qiyang Li and Sergey Levine},
  title = {Flow {Q}-Learning},
  booktitle = {International Conference on Machine Learning (ICML)},
  year = {2025},
}

@inproceedings{park2024bottleneck,
  author = {Seohong Park and Kevin Frans and Sergey Levine and Aviral Kumar},
  title = {Is Value Learning Really the Main Bottleneck in Offline {RL}?},
  booktitle = {Advances in Neural Information Processing Systems (NeurIPS)},
  year = {2024},
}

@inproceedings{wagenmaker2025dsrl,
  title={Steering Your Diffusion Policy with Latent Space Reinforcement Learning},
  author={Wagenmaker, Andrew and Zhang, Yunchu and Nakamoto, Mitsuhiko and Park, Seohong and Yagoub, Waleed and Nagabandi, Anusha and Gupta, Abhishek and Levine, Sergey},
  booktitle={Conference on Robot Learning},
  pages={258--282},
  year={2025},
  organization={PMLR}
}

@article{dong2025expo,
  title={Expo: Stable reinforcement learning with expressive policies},
  author={Dong, Perry and Li, Qiyang and Sadigh, Dorsa and Finn, Chelsea},
  journal={arXiv preprint arXiv:2507.07986},
  year={2025}
}

@article{attarian2026ufops,
  title={Update-Free On-Policy Steering via Verifiers},
  author={Attarian, Maria and Vyse, Ian and Voelcker, Claas and Gerigk, Jasper and Opryshko, Evgenii and Almasri, Anas and Singh, Sumeet and Du, Yilun and Gilitschenski, Igor},
  journal={arXiv preprint arXiv:2603.10282},
  year={2026}
}

@inproceedings{du2025dynaguide,
  title={DynaGuide: Steering Diffusion Polices with Active Dynamic Guidance},
  author={Du, Max and Song, Shuran},
  year = {2025},
  booktitle={The Thirty-ninth Annual Conference on Neural Information Processing Systems}
}

@inproceedings{nakamoto2024vgps,
  author = {Mitsuhiko Nakamoto and Oier Mees and Aviral Kumar and Sergey Levine},
  title = {Steering Your Generalists: Improving Robotic Foundation Models via Value Guidance},
  booktitle = {Conference on Robot Learning (CoRL)},
  pages = {4996--5013},
  year = {2024},
}

@inproceedings{bansal2024universal,
  author = {Arpit Bansal and Hong-Min Chu and Avi Schwarzschild and Soumyadip Sengupta and Micah Goldblum and Jonas Geiping and Tom Goldstein},
  title = {Universal Guidance for Diffusion Models},
  booktitle = {International Conference on Learning Representations (ICLR)},
  year = {2024},
}

@article{potaptchik2026mfm,
  title={Meta Flow Maps enable scalable reward alignment},
  author={Potaptchik, Peter and Saravanan, Adhi and Mammadov, Abbas and Prat, Alvaro and Albergo, Michael S and Teh, Yee Whye},
  journal={arXiv preprint arXiv:2601.14430},
  year={2026}
}

@inproceedings{lipman2023flow,
  author = {Yaron Lipman and Ricky T. Q. Chen and Heli Ben-Hamu and Maximilian Nickel and Matthew Le},
  title = {Flow Matching for Generative Modeling},
  booktitle = {International Conference on Learning Representations (ICLR)},
  year = {2023},
}

@inproceedings{liu2023rectifiedflow,
  author = {Xingchao Liu and Chengyue Gong and Qiang Liu},
  title = {Flow Straight and Fast: Learning to Generate and Transfer Data with Rectified Flow},
  booktitle = {International Conference on Learning Representations (ICLR)},
  year = {2023},
}

@article{chi2025diffusionpolicy,
  author = {Cheng Chi and Zhenjia Xu and Siyuan Feng and Eric Cousineau and Yilun Du and Benjamin Burchfiel and Russ Tedrake and Shuran Song},
  title = {Diffusion Policy: Visuomotor Policy Learning via Action Diffusion},
  journal = {International Journal of Robotics Research},
  volume = {44},
  pages = {1684--1704},
  year = {2025},
}

@article{black2024pi0,
  title={$\pi_0$: A Vision-Language-Action Flow Model for General Robot Control},
  author={Black, Kevin and Brown, Noah and Driess, Danny and Esmail, Adnan and Equi, Michael and Finn, Chelsea and Fusai, Niccolo and Groom, Lachy and Hausman, Karol and Ichter, Brian and others},
  journal={arXiv preprint arXiv:2410.24164},
  year={2024}
}

@inproceedings{kostrikov2022iql,
  author = {Ilya Kostrikov and Ashvin Nair and Sergey Levine},
  title = {Offline Reinforcement Learning with Implicit {Q}-Learning},
  booktitle = {International Conference on Learning Representations (ICLR)},
  year = {2022},
}

@inproceedings{fujimoto2021td3bc,
  author = {Scott Fujimoto and Shixiang Shane Gu},
  title = {A Minimalist Approach to Offline Reinforcement Learning},
  booktitle = {Advances in Neural Information Processing Systems (NeurIPS)},
  pages = {20132--20145},
  year = {2021},
}

@inproceedings{dhariwal2021classifier,
  author = {Prafulla Dhariwal and Alexander Quinn Nichol},
  title = {Diffusion Models Beat {GANs} on Image Synthesis},
  booktitle = {Advances in Neural Information Processing Systems (NeurIPS)},
  pages = {8780--8794},
  year = {2021},
}

@inproceedings{haarnoja2018sac,
  author = {Tuomas Haarnoja and Aurick Zhou and Pieter Abbeel and Sergey Levine},
  title = {Soft Actor-Critic: Off-Policy Maximum Entropy Deep Reinforcement Learning with a Stochastic Actor},
  booktitle = {International Conference on Machine Learning (ICML)},
  pages = {1856--1865},
  year = {2018},
}

@inproceedings{song2021scorebased,
  author = {Yang Song and Jascha Sohl-Dickstein and Diederik P. Kingma and Abhishek Kumar and Stefano Ermon and Ben Poole},
  title = {Score-Based Generative Modeling through Stochastic Differential Equations},
  booktitle = {International Conference on Learning Representations (ICLR)},
  year = {2021},
}

@inproceedings{wang2023dql,
  author = {Zhendong Wang and Jonathan J. Hunt and Mingyuan Zhou},
  title = {Diffusion Policies as an Expressive Policy Class for Offline Reinforcement Learning},
  booktitle = {International Conference on Learning Representations (ICLR)},
  year = {2023},
}

@inproceedings{ding2024consistency,
  author = {Zihan Ding and Chi Jin},
  title = {Consistency Models as a Rich and Efficient Policy Class for Reinforcement Learning},
  booktitle = {International Conference on Learning Representations (ICLR)},
  year = {2024},
}

@inproceedings{espinosadice2025shortcut,
  title={Scaling Offline RL via Efficient and Expressive Shortcut Models},
  author={Espinosa-Dice, Nicolas and Zhang, Yiyi and Chen, Yiding and Guo, Bradley and Oertell, Owen and Swamy, Gokul and Brantley, Kiant{\'e} and Sun, Wen},
  booktitle={The Thirty-ninth Annual Conference on Neural Information Processing Systems},
  year = {2025},
}

@inproceedings{psenka2024qsm,
  author = {Michael Psenka and Alejandro Escontrela and Pieter Abbeel and Yi Ma},
  title = {Learning a Diffusion Model Policy from Rewards via {Q}-Score Matching},
  booktitle = {International Conference on Machine Learning (ICML)},
  year = {2024},
}

@inproceedings{fang2025dac,
  author = {Linjiajie Fang and Ruoxue Liu and Jing Zhang and Wenjia Wang and Bingyi Jing},
  title = {Diffusion Actor-Critic: Formulating Constrained Policy Iteration as Diffusion Noise Regression for Offline Reinforcement Learning},
  booktitle = {International Conference on Learning Representations (ICLR)},
  year = {2025},
}

@article{frans2025cfgrl,
  title={Diffusion guidance is a controllable policy improvement operator},
  author={Frans, Kevin and Park, Seohong and Abbeel, Pieter and Levine, Sergey},
  journal={arXiv preprint arXiv:2505.23458},
  year={2025}
}

@article{hansenestruch2023idql,
  title={Idql: Implicit q-learning as an actor-critic method with diffusion policies},
  author={Hansen-Estruch, Philippe and Kostrikov, Ilya and Janner, Michael and Kuba, Jakub Grudzien and Levine, Sergey},
  journal={arXiv preprint arXiv:2304.10573},
  year={2023}
}

@inproceedings{singh2021parrot,
  author = {Avi Singh and Huihan Liu and Gaoyue Zhou and Albert Yu and Nicholas Rhinehart and Sergey Levine},
  title = {Parrot: Data-Driven Behavioral Priors for Reinforcement Learning},
  booktitle = {International Conference on Learning Representations (ICLR)},
  year = {2021},
}

@inproceedings{yuan2025decorator,
  author = {Xiu Yuan and Tongzhou Mu and Stone Tao and Yunhao Fang and Mengke Zhang and Hao Su},
  title = {Policy Decorator: Model-Agnostic Online Refinement for Large Policy Model},
  booktitle = {International Conference on Learning Representations (ICLR)},
  year = {2025},
}

@article{mark2025policyagnostic,
  title={Policy agnostic rl: Offline rl and online rl fine-tuning of any class and backbone},
  author={Mark, Max Sobol and Gao, Tian and Sampaio, Georgia Gabriela and Srirama, Mohan Kumar and Sharma, Archit and Finn, Chelsea and Kumar, Aviral},
  journal={arXiv preprint arXiv:2412.06685},
  year={2024}
}

@inproceedings{ren2024diffusion,
  title={Diffusion Policy Policy Optimization},
  author={Ren, Allen Z and Lidard, Justin and Ankile, Lars Lien and Simeonov, Anthony and Agrawal, Pulkit and Majumdar, Anirudha and Burchfiel, Benjamin and Dai, Hongkai and Simchowitz, Max},
  booktitle={International Conference on Learning Representations (ICLR)},
  year={2025}
}

@article{mcallister2025flow,
  title={Flow matching policy gradients},
  author={McAllister, David and Ge, Songwei and Yi, Brent and Kim, Chung Min and Weber, Ethan and Choi, Hongsuk and Feng, Haiwen and Kanazawa, Angjoo},
  journal={arXiv preprint arXiv:2507.21053},
  year={2025}
}

@inproceedings{black2025pi05,
  title={$\pi_{0.5}$: a Vision-Language-Action Model with Open-World Generalization},
  author={Black, Kevin and Brown, Noah and Darpinian, James and Dhabalia, Karan and Driess, Danny and Esmail, Adnan and Equi, Michael Robert and Finn, Chelsea and Fusai, Niccolo and Galliker, Manuel Y and others},
  booktitle={9th Annual Conference on Robot Learning},
  year={2025}
}

@article{liu2023libero,
  title={Libero: Benchmarking knowledge transfer for lifelong robot learning},
  author={Liu, Bo and Zhu, Yifeng and Gao, Chongkai and Feng, Yihao and Liu, Qiang and Zhu, Yuke and Stone, Peter},
  journal={Advances in Neural Information Processing Systems},
  volume={36},
  pages={44776--44791},
  year={2023}
}

@article{li2025three,
  title={The Three Regimes of Offline-to-Online Reinforcement Learning},
  author={Li, Lu and Ni, Tianwei and Sun, Yihao and Bacon, Pierre-Luc},
  journal={arXiv preprint arXiv:2510.01460},
  year={2025}
}

@inproceedings{lee2022offline,
  title={Offline-to-online reinforcement learning via balanced replay and pessimistic q-ensemble},
  author={Lee, Seunghyun and Seo, Younggyo and Lee, Kimin and Abbeel, Pieter and Shin, Jinwoo},
  booktitle={Conference on Robot Learning},
  pages={1702--1712},
  year={2022},
  organization={PMLR}
}

@inproceedings{song2022hybrid,
  title={Hybrid RL: Using both offline and online data can make RL efficient},
  author={Song, Yuda and Zhou, Yifei and Sekhari, Ayush and Bagnell, Drew and Krishnamurthy, Akshay and Sun, Wen},
  booktitle={The Eleventh International Conference on Learning Representations},
  year={2023}
}

@article{nakamoto2023calql,
  title={Cal-ql: Calibrated offline rl pre-training for efficient online fine-tuning},
  author={Nakamoto, Mitsuhiko and Zhai, Simon and Singh, Anikait and Sobol Mark, Max and Ma, Yi and Finn, Chelsea and Kumar, Aviral and Levine, Sergey},
  journal={Advances in Neural Information Processing Systems},
  volume={36},
  pages={62244--62269},
  year={2023}
}

@inproceedings{li2025qc,
  title={Reinforcement Learning with Action Chunking},
  author={Li, Qiyang and Zhou, Zhiyuan and Levine, Sergey},
  booktitle={The Thirty-ninth Annual Conference on Neural Information Processing Systems},
  year={2025}
}

@inproceedings{janner2022planning,
  title={Planning with Diffusion for Flexible Behavior Synthesis},
  author={Janner, Michael and Du, Yilun and Tenenbaum, Joshua and Levine, Sergey},
  booktitle={International Conference on Machine Learning},
  pages={9902--9915},
  year={2022},
  organization={PMLR}
}

@inproceedings{ajay2022conditional,
  title={Is Conditional Generative Modeling all you need for Decision Making?},
  author={Ajay, Anurag and Du, Yilun and Gupta, Abhi and Tenenbaum, Joshua B and Jaakkola, Tommi S and Agrawal, Pulkit},
  booktitle={The Eleventh International Conference on Learning Representations},
  year={2023},
}

@article{kang2023efficient,
  title={Efficient diffusion policies for offline reinforcement learning},
  author={Kang, Bingyi and Ma, Xiao and Du, Chao and Pang, Tianyu and Yan, Shuicheng},
  journal={Advances in Neural Information Processing Systems},
  volume={36},
  pages={67195--67212},
  year={2023}
}

@article{he2023diffcps,
  title={Diffcps: Diffusion model based constrained policy search for offline reinforcement learning},
  author={He, Longxiang and Shen, Li and Zhang, Linrui and Tan, Junbo and Wang, Xueqian},
  journal={arXiv preprint arXiv:2310.05333},
  year={2023}
}

@inproceedings{lu2023contrastive,
  title={Contrastive energy prediction for exact energy-guided diffusion sampling in offline reinforcement learning},
  author={Lu, Cheng and Chen, Huayu and Chen, Jianfei and Su, Hang and Li, Chongxuan and Zhu, Jun},
  booktitle={International Conference on Machine Learning},
  pages={22825--22855},
  year={2023},
  organization={PMLR}
}

@article{chen2024diffusion,
  title={Diffusion policies creating a trust region for offline reinforcement learning},
  author={Chen, Tianyu and Wang, Zhendong and Zhou, Mingyuan},
  journal={Advances in Neural Information Processing Systems},
  volume={37},
  pages={50098--50125},
  year={2024}
}

@article{zhang2024entropy,
  title={Entropy-regularized diffusion policy with q-ensembles for offline reinforcement learning},
  author={Zhang, Ruoqi and Luo, Ziwei and Sj{\"o}lund, Jens and Sch{\"o}n, Thomas B and Mattsson, Per},
  journal={Advances in neural information processing systems},
  volume={37},
  pages={98871--98897},
  year={2024}
}

@inproceedings{ankile2025imitation,
  title={From imitation to refinement-residual rl for precise assembly},
  author={Ankile, Lars and Simeonov, Anthony and Shenfeld, Idan and Torne, Marcel and Agrawal, Pulkit},
  booktitle={2025 IEEE International Conference on Robotics and Automation (ICRA)},
  pages={01--08},
  year={2025},
  organization={IEEE}
}

@inproceedings{johannink2019residual,
  title={Residual reinforcement learning for robot control},
  author={Johannink, Tobias and Bahl, Shikhar and Nair, Ashvin and Luo, Jianlan and Kumar, Avinash and Loskyll, Matthias and Ojea, Juan Aparicio and Solowjow, Eugen and Levine, Sergey},
  booktitle={2019 international conference on robotics and automation (ICRA)},
  pages={6023--6029},
  year={2019},
  organization={IEEE}
}

@article{ho2022classifier,
  title={Classifier-free diffusion guidance},
  author={Ho, Jonathan and Salimans, Tim},
  journal={arXiv preprint arXiv:2207.12598},
  year={2022}
}

@article{xie2021policy,
  title={Policy finetuning: Bridging sample-efficient offline and online reinforcement learning},
  author={Xie, Tengyang and Jiang, Nan and Wang, Huan and Xiong, Caiming and Bai, Yu},
  journal={Advances in Neural Information Processing Systems},
  volume={34},
  pages={27395--27407},
  year={2021}
}

@inproceedings{zhang2025reinflow,
  title={ReinFlow: Fine-tuning Flow Matching Policy with Online Reinforcement Learning},
  author={Zhang, Tonghe and Yu, Chao and Su, Sichang and Wang, Yu},
  year = {2025},
  booktitle={The Thirty-ninth Annual Conference on Neural Information Processing Systems}
}

@inproceedings{zhou2024wsrl,
  title={Efficient Online Reinforcement Learning Fine-Tuning Need Not Retain Offline Data},
  author={Zhou, Zhiyuan and Peng, Andy and Li, Qiyang and Levine, Sergey and Kumar, Aviral},
  year = {2024},
  booktitle={The Thirteenth International Conference on Learning Representations}
}

@article{song2021maximum,
  title={Maximum likelihood training of score-based diffusion models},
  author={Song, Yang and Durkan, Conor and Murray, Iain and Ermon, Stefano},
  journal={Advances in Neural Information Processing Systems},
  volume={34},
  pages={1415--1428},
  year={2021}
}

@article{holderrieth2026diamond,
  title={Diamond Maps: Efficient Reward Alignment via Stochastic Flow Maps},
  author={Holderrieth, Peter and Chen, Douglas and Eyring, Luca and Shah, Ishin and Anantharaman, Giri and He, Yutong and Akata, Zeynep and Jaakkola, Tommi and Boffi, Nicholas Matthew and Simchowitz, Max},
  journal={arXiv preprint arXiv:2602.05993},
  year={2026}
}

@inproceedings{wang2026meanflowql,
  title={One-step generative policies with q-learning: A reformulation of meanflow},
  author={Wang, Zeyuan and Li, Da and Chen, Yulin and Shi, Ye and Bai, Liang and Yu, Tianyuan and Fu, Yanwei},
  booktitle={Proceedings of the AAAI Conference on Artificial Intelligence},
  volume={40},
  number={31},
  pages={26751--26759},
  year={2026}
}

@inproceedings{wu2023tds,
  author = {Luhuan Wu and Brian L. Trippe and Christian A. Naesseth and David M. Blei and John P. Cunningham},
  title = {Practical and Asymptotically Exact Conditional Sampling in Diffusion Models},
  booktitle = {Advances in Neural Information Processing Systems (NeurIPS)},
  year = {2023},
}

\newpage
\appendix

\section{Contributions}
\label{app:contributions}

Author Contributions: Y.R. and A.B. conceived the project. Y.R. developed the core method, ran OGBench experiments, \qpilotm LIBERO experiments and supported \qpilotu LIBERO experiments. C.C led \qpilotu in LIBERO, baselines in LIBERO, and the ablation study in LIBERO. Y.R. and C.C. made figures. N.R. provided high-level project guidance. Y.R. led paper writing with input from all authors.

\section{Proofs}
\label{app:proofs}

We restate and prove Proposition~\ref{prop:convergence} from Section~\ref{sec:method:theory}.

\begin{proposition}[Convergence of Q-steered MFM sampling]
\label{prop:convergence}
Fix a state $s \in \cS$.
Let $\hat{p}_1^s$ denote the distribution over actions produced by Algorithm~\ref{alg:mfm-steer} (with the $\sigma_t^2/2$ formulation in place of gradient rescaling) using $K$ uniform Euler steps and $N$ independent MFM posterior samples per step.
Suppose:
\begin{enumerate}
  \item The tilt scalar $\tau \bar{Q}(s, \cdot)$ satisfies $\tau \bar{Q}(s, \cdot) \in C^2(\R^A)$ with the function and its first two derivatives all bounded. In practice we obtain a bounded tilt scalar by clipping the (batch-normalized) Q-value to a fixed range; see remark below.
  \item The MFM $\hat{X}_{0,1}(\epsilon; t, x)$ and the base velocity $v_\theta(x, t, s)$ are jointly $C^2$ in $(t, x)$ on $[0,1] \times \R^A$, with all first- and second-order partial derivatives bounded.
  \item The diffusion coefficient satisfies $0 < \sigma_{\min} \leq \sigma_t \leq \sigma_{\max}$ for all $t \in [0, 1]$.
\end{enumerate}
Then the optimal steered velocity $v_\theta^\star(x, t, s) = v_\theta(x, t, s) + \frac{\sigma_t^2}{2} \nabla V_t(x, s)$ is Lipschitz in $x$ and H\"{o}lder continuous in $t$, and there exists a constant $C > 0$ independent of $K$ and $N$ such that
\begin{equation}
  W_2\!\left(\hat{p}_1^s,\, \pi^\star(\cdot | s)\right) \leq C \left( \frac{1}{\sqrt{K}} + \frac{1}{\sqrt{N}} \right), \qquad
  \mathrm{KL}\!\left(\hat{p}_1^s \,\|\, \pi^\star(\cdot | s)\right) \leq C \left( \frac{1}{K} + \frac{1}{N} \right).
  \label{eq:convergence}
\end{equation}
\end{proposition}

\begin{proof}
For brevity, write $\psi(\cdot) := \tau \bar{Q}(s, \cdot)$ \emph{locally within this proof}; this is the tilt scalar of \citet{potaptchik2026mfm}'s framework instantiated to our setting, not the MDP reward of Section~\ref{sec:prelim:rl}.
The convergence rates follow from Proposition C.2 of \citet{potaptchik2026mfm} by setting their tilt scalar to $r$.
It suffices to verify their regularity conditions.
Their assumption 1 ($r \in C^1$ with $r$ and $\nabla r$ bounded) follows from our assumption 1, which additionally requires bounded second derivatives so that $\nabla V_t$ inherits a Lipschitz constant via the chain rule.
Their assumption 2 (MFM $C^1$ with bounded Jacobian) is immediate from our assumption 2.
Their assumption 4 (bounded diffusion) is our assumption 3.

Their assumption 3 requires the optimal steered velocity $v_\theta^\star = v_\theta + \frac{\sigma_t^2}{2} \nabla V_t$ to be $L$-Lipschitz in $x$ and $\frac{1}{2}$-H\"{o}lder in $t$.
We show this follows from our assumptions 1-2.
The log-tilt potential is $V_t(x) = \log \E[\exp(\psi(\hat{X}_{0,1}(\epsilon; t, x)))] = \log \E[\exp(\tau \bar{Q}(s, \hat{X}_{0,1}(\epsilon; t, x)))]$.
Since $\bar{Q}(s, \cdot) \in C^2$ with bounded derivatives and $\hat{X}_{0,1}(\cdot; t, \cdot) \in C^2$ with bounded derivatives, the composition is $C^2$ in $x$ with bounded derivatives by the chain rule.
Taking the logarithm of the expectation (a smooth, monotone operation on the cumulant generating function) preserves this regularity, so $\nabla V_t$ is Lipschitz in $x$.
For Hölder continuity in $t$, assumption 2 gives bounded $\partial_t \hat{X}_{0,1}$ and $\partial_t v_\theta$, and assumption 3 gives a bounded $\sigma_t^2$ that is $C^1$ in $t$ (e.g.\ the standard schedule $\sigma_t^2 = 2(1-t)/t$ on the truncated interval $[\varepsilon, 1-\varepsilon]$); differentiating $V_t$ in $t$ under the expectation (justified by dominated convergence, using assumption 1) yields a bounded $\partial_t \nabla V_t$. Hence $\partial_t v_\theta^\star$ is bounded, so $v_\theta^\star$ is Lipschitz in $t$ on the truncated interval and a fortiori $\tfrac{1}{2}$-Hölder.
\end{proof}

The bound decomposes into two sources of error: time-discretization error ($1/\sqrt{K}$ in $W_2$, $1/K$ in KL), which is standard for Euler-Maruyama schemes, and Monte Carlo estimation error ($1/\sqrt{N}$ in $W_2$, $1/N$ in KL), arising from approximating $\nabla V_t$ with finitely many posterior samples.
Crucially, both errors are controlled independently. Increasing the number of posterior samples $N$ improves the gradient estimate without requiring finer time discretization, and vice versa.

\paragraph{Assumptions in Practice.}
Assumption 2 holds for neural networks with smooth activations (e.g., GeLU, SiLU) and bounded weights, which is enforced by gradient clipping during training.
Assumption 1 is more delicate: a learned critic $\bar{Q}(s, \cdot)$ is not bounded a priori, and $\|\tau \bar{Q}(s, \cdot)\|_\infty$ enters the constant $C$ exponentially through the moment bounds in Proposition C.4 of \citet{potaptchik2026mfm}. We satisfy the assumption operationally by batch-normalizing $\bar{Q}$ across the minibatch before exponentiating and clipping $\tau \cdot \bar{Q}_{\mathrm{norm}}$ to a fixed interval (Section~\ref{sec:method:training}), which keeps the tilt scalar bounded with constants independent of the absolute Q-scale. Strictly speaking, batch normalization makes the tilt scalar depend on the empirical batch distribution and therefore non-stationary across training, which is outside the iid setting analyzed in Proposition C.4; we treat the resulting bound as a per-batch guarantee rather than a uniform one.
Assumption 3 requires a bounded diffusion schedule; the theoretical schedule $\sigma_t^2 = 2(1-t)/t$ violates the upper bound as $t \to 0$ (where $\sigma_t \to \infty$) and the lower bound as $t \to 1$ (where $\sigma_t \to 0$). We skip steering at $t = 0$ (Section~\ref{sec:method:framework}) to control the upper end, and--following standard practice~\citep{potaptchik2026mfm}--the formal result applies to the process stopped at $t = 1 - \varepsilon$ to control the lower end.
We note that the constant $C$ depends on the Lipschitz constants of the critic and MFM and may scale exponentially with the action dimension~\citep{potaptchik2026mfm}.
The gradient rescaling used in our practical algorithm (\cref{eq:grad_rescale}) replaces the $\sigma_t^2/2$ factor with a drift-magnitude-matching heuristic, which targets the same tilted distribution but falls outside the formal SDE analysis; we find it to be more stable empirically.

\section{Implementation Details}
\label[appendix]{app:implementation}

The per-step computational overhead of inference-time steering (Algorithm~\ref{alg:mfm-steer}) is a single forward-backward through $v_\theta$ and the critic for \qpilotu, and $K$ forward passes through the MFM and critic plus one backward pass for \qpilotm--all single-step evaluations with no inner rollouts.

\subsection{Architecture}
\label{app:implementation:arch}

In OGBench experiments, all networks are MLPs with hidden dimensions $(512, 512, 512, 512)$ and SiLU activations.
The base velocity field $v_\theta(x_t, t, s)$ is an \texttt{ActorVectorField} that concatenates $[s, x_t, \mathrm{FF}(t)]$ and outputs a velocity in $\R^{H \cdot A}$, where $\mathrm{FF}(\cdot)$ is a 64-dim Fourier-feature embedding and $H$ is the action-chunking horizon.
The critic is a $J{=}10$-ensemble of \texttt{Value} networks (LayerNorm enabled) with pessimism weight $\rho$ (Section~\ref{sec:method:framework}).
The Meta Flow Map $\hat{v}_{u,w}(\bar{x};\, t, x_t, s)$ is an \texttt{MFMFlowMap}: it concatenates $[s, \bar{x}, x_t, \mathrm{FF}(t), \mathrm{FF}(u), \mathrm{FF}(w)]$, with three independent Fourier embeddings for the outer time $t$ and the auxiliary times $(u, w)$.
LayerNorm is disabled inside the MFM and base velocity field; its conditioning-action input is passed through unchanged (no $t$-gating).
We instantiate Polyak-averaged target networks for the critic, the base flow, and the MFM.

In LIBERO experiments, the critic $Q(s,a)$ uses double-view $(64, 64)$ images from the global camera and the wrist camera as the pixel inputs. Its network consists of a 4-layer CNN encoder and MLPs with hidden dimensions $(128, 128, 128)$. We use the same 10-network Q-value ensemble as the critic in OGBench experiments. For baselines like DSRL or EXPO that require an additional Gaussian policy, we use the same structure as the critic and follow the hyperparameters used in OGBench.

\subsection{Training}
\label{app:implementation:training}

For OGBench, all parameters (base flow $v_\theta$, MFM $\hat{v}_\xi$, and critic ensemble $\{Q_{\phi_j}\}$) are jointly optimized with a single Adam optimizer at learning rate $3 \times 10^{-4}$ and global gradient clipping at norm $1.0$. No learning-rate schedule or warmup is used.
Targets are updated by Polyak averaging with $\tau_{\mathrm{critic}} = \tau_{\mathrm{flow}} = \tau_{\mathrm{mfm}} = 0.005$.
Each gradient step samples a length-$H$ sequence batch of size $256$ from the offline dataset (offline phase) or from the union of offline data and online replay (online phase, one gradient update per one environment interaction, no balanced sampling).
Cross-episode transitions are masked out via the \texttt{valid} mask of \citet{li2026qam}.

The total loss is
\[
  \mathcal{L} = \underbrace{\mathcal{L}_{\mathrm{critic}}}_{\text{TD}} + \underbrace{\mathcal{L}_{\mathrm{FM}}}_{\text{base flow, \cref{eq:fm_loss}}} + \lambda_{\mathrm{diag}} \mathcal{L}_{\mathrm{diag}} + \lambda_{\mathrm{cons}} \mathcal{L}_{\mathrm{cons}},
\]
with $\lambda_{\mathrm{diag}} = 1.0$, $\lambda_{\mathrm{cons}} = 0.5$ for \qpilotm and $\lambda_{\mathrm{diag}} = \lambda_{\mathrm{cons}} = 0$ for \qpilotu (no MFM is trained).
$\mathcal{L}_{\mathrm{critic}}$ uses TD against the pessimistic ensemble target with the next action sampled from the steered policy (\texttt{critic\_steered\_targets}).
$\mathcal{L}_{\mathrm{diag}}$ samples the outer time as $t = u^2,\, u \sim \mathcal{U}[0,1]$ to bias supervision toward small $t$ where the posterior is broader (\texttt{mfm\_t\_outer\_power}=$2$).
$\mathcal{L}_{\mathrm{cons}}$ samples ordered times $0 \leq u < m < w \leq 1$ by sorting two uniforms and taking $m = u + \gamma(w - u)$ with $\gamma \sim \mathcal{U}[0,1]$, and supervises the direct map against the composed map computed under target parameters.
Both MFM losses use the adaptive reweighting of \citet{potaptchik2026mfm} (\texttt{adaptive\_loss}=true) with $p_{\mathrm{diag}} = 0.5$, $p_{\mathrm{cons}} = 1.0$, $c = 0.01$.
We do not use the FT auxiliary loss in our experiments ($\lambda_{\mathrm{ft}} = 0$).

\paragraph{LIBERO MFM Distillation.}
For LIBERO, the auxiliary meta flow map is pretrained offline against the frozen $\pi_{0.5}$-LIBERO checkpoint and the critic is then trained online (the base flow is never updated). The MFM $\hat{v}_{u,w}$ uses a transformer-style architecture matching the $\pi_{0.5}$ action expert: hidden width $512$, $8$ blocks, and $8$ attention heads (\texttt{mfm\_width}=$512$, \texttt{mfm\_num\_blocks}=$8$, \texttt{mfm\_num\_heads}=$8$). Following \citet{potaptchik2026mfm}, the MFM is initialized from the $\pi_{0.5}$-LIBERO action-expert weights with the auxiliary-time gating zero-initialized, so that at the start of distillation the network reproduces the frozen base flow exactly and only learns the off-diagonal posterior structure.
Distillation follows the GLASS reparameterization of \citet{potaptchik2026mfm}: the frozen $\pi_{0.5}$ acts as the conditional teacher, supplying noisy-state/clean-action pairs from LIBERO-90 demonstrations, and the MFM is trained with the same diagonal and consistency objectives ($\lambda_{\mathrm{diag}} = 1.0$, $\lambda_{\mathrm{cons}} = 0.5$, adaptive reweighting on) and the same outer-time power-2 sampling as in OGBench.
We use batch size $256$ for $4 \times 10^4$ gradient steps, optimized with RAdam at learning rate $1 \times 10^{-4}$ with a linear warmup over $2000$ steps and a constant schedule thereafter, and global gradient clipping at norm $1.5$. The MFM target is updated by Polyak averaging with $\tau_{\mathrm{mfm}} = 10^{-4}$ (matching \citet{potaptchik2026mfm}'s reference value, which is more conservative than the OGBench setting since the MFM is updated against a frozen teacher).
The critic uses a separate Adam optimizer at $3\times10^{-4}$ with a $500$-step warmup and cosine decay over $3 \times 10^5$ steps to $0.1\times$ peak, with global gradient clipping at norm $1.0$ and Polyak rate $\tau_{\mathrm{critic}} = 0.005$.

\Cref{alg:qpilots-train} summarizes the resulting offline-to-online procedure for the OGBench experiments. \qpilotu uses the same outer loop with $\lambda_{\mathrm{diag}} = \lambda_{\mathrm{cons}} = 0$ and no MFM network, in which case the steered policy in \cref{alg:mfm-steer} replaces $\widehat{\nabla V}_t^{\mathrm{MFM}}$ with the Tweedie estimator $\widehat{\nabla V}_t^{\mathrm{UG}}$ of \cref{eq:ug_grad}.

\begin{algorithm}[t]
\caption{\qpilotm: Offline-to-Online Training (OGBench)}
\label{alg:qpilots-train}
\begin{algorithmic}[1]
\REQUIRE Offline data $\mathcal{D}_{\mathrm{off}}$, environment, offline steps $M_{\mathrm{off}}$, online env steps $M_{\mathrm{on}}$, batch size $B$, target rate $\tau$, MFM weights $\lambda_{\mathrm{diag}},\lambda_{\mathrm{cons}}$, steering coefficient $\alpha$, action-chunk horizon $H$
\STATE Initialize base flow $v_\theta$, critic ensemble $\{Q_{\phi_j}\}_{j=1}^{J}$, MFM $\hat{v}_\xi$ (\qpilotm only), and target copies $(\theta', \phi', \xi')$;\quad replay $\mathcal{R} \leftarrow \emptyset$
\STATE \textbf{// Offline phase}
\FOR{$k = 1, \ldots, M_{\mathrm{off}}$}
    \STATE Sample length-$H$ sequence batch $\mathcal{B} \sim \mathcal{D}_{\mathrm{off}}$ of size $B$;\quad call \textsc{Update}$(\mathcal{B})$
\ENDFOR
\STATE \textbf{// Online phase (one gradient step per environment step)}
\STATE Reset environment: $s \leftarrow s_0$
\FOR{$k = 1, \ldots, M_{\mathrm{on}}$}
    \STATE Sample chunk $a_{1:H} \sim \pi^{\alpha}(\cdot \mid s)$ via \cref{alg:mfm-steer} \COMMENT{Steered rollout}
    \STATE Execute $a_{1:H}$ open-loop; push transitions to $\mathcal{R}$; advance $s$ (reset on terminal)
    \STATE Sample $\mathcal{B} \sim \mathcal{D}_{\mathrm{off}} \cup \mathcal{R}$ of size $B$;\quad call \textsc{Update}$(\mathcal{B})$
\ENDFOR
\STATE \textbf{return} $v_\theta$, $\{Q_{\phi_j}\}_{j=1}^{J}$, $\hat{v}_\xi$
\STATE
\STATE \textbf{procedure} \textsc{Update}$(\mathcal{B})$
\STATE \quad For each $(s, a, r, s') \in \mathcal{B}$: sample $a' \sim \pi^{\alpha}(\cdot \mid s')$ via \cref{alg:mfm-steer} using target nets
\STATE \quad TD target: $y = r + \gamma\bigl(\tfrac{1}{J}\sum_j Q_{\phi'_j}(s', a') - \rho\,\mathrm{std}_j\, Q_{\phi'_j}(s', a')\bigr)$
\STATE \quad $\mathcal{L}_{\mathrm{critic}} = \tfrac{1}{J}\sum_j \bigl(Q_{\phi_j}(s, a) - \mathrm{sg}(y)\bigr)^2$ \COMMENT{Pessimistic ensemble TD}
\STATE \quad $\mathcal{L}_{\mathrm{FM}}$ from \cref{eq:fm_loss};\quad $\mathcal{L}_{\mathrm{diag}}, \mathcal{L}_{\mathrm{cons}}$ from \cref{eq:diag_loss,eq:consistency_loss}
\STATE \quad $\mathcal{L} = \mathcal{L}_{\mathrm{critic}} + \mathcal{L}_{\mathrm{FM}} + \lambda_{\mathrm{diag}}\mathcal{L}_{\mathrm{diag}} + \lambda_{\mathrm{cons}}\mathcal{L}_{\mathrm{cons}}$
\STATE \quad Joint Adam step on $(\theta, \{\phi_j\}, \xi)$ with global gradient clipping at norm $1.0$
\STATE \quad Polyak: $\theta' \leftarrow (1-\tau)\theta' + \tau\theta$ (similarly for $\phi'_j$ and $\xi'$)
\STATE \textbf{end procedure}
\end{algorithmic}
\end{algorithm}

\subsection{Inference}
\label{app:implementation:inference}

Action generation uses $K{=}10$ Euler steps and best-of-$1$ selection (\texttt{flow\_steps}=$10$, \texttt{best\_of\_n}=$1$).
For \qpilotm we draw $N{=}8$ posterior samples per Euler step (\texttt{num\_mfm\_samples}=$8$) in OGBench and $N{=}4$ posterior samples per Euler step (\texttt{num\_mfm\_samples}=$4$) in LIBERO.
Both \qpilotm and \qpilotu use gradient rescaling (\cref{eq:grad_rescale}) rather than the $\sigma_t^2/2$ schedule, with steering skipped at $t = 0$.
For \qpilotm the steering gradient is computed against the online MFM and the target critic; for \qpilotu the gradient is computed against the target base flow and the target critic, which we found to be more stable for the biased Tweedie estimate.
At inference each $H$-action chunk is executed open-loop in the environment.

\subsection{Per-Domain Settings}
\label{app:implementation:hparams}

\paragraph{OGBench Per-Task Settings and Results.} Domain-level hyperparameters follow \citet{li2026qam} and are summarized in Table~\ref{tab:domain_hparams}.
The steering coefficient $\alpha$ was selected per domain by sweeping $\{0.1, 0.2, 0.3, 0.5, 1.0, 2.0\}$ on two tuning tasks per domain (tasks 1 and 4 for navigation, tasks 2 and 4 for manipulation) with 4 seeds per configuration, mirroring the protocol of \citet{li2026qam}. The combined performance across the two tuning tasks selects $\alpha$ for the domain, which is then reused across all five tasks and evaluated on 8 fresh seeds. Per-domain sweep curves are shown in \cref{fig:alpha_sweep}.
The selected coefficients are: \texttt{antmaze-large}: $0.2$, \texttt{antmaze-giant}: $0.3$, \texttt{humanoidmaze-medium}: $0.2$, \texttt{humanoidmaze-large}: $0.2$, \texttt{cube-double}: $0.3$, \texttt{cube-triple}: $0.5$, \texttt{cube-quadruple}: $0.5$, \texttt{scene-play-sparse}: $0.2$, \texttt{puzzle-3x3-play-sparse}: $0.3$, \texttt{puzzle-4x4-play-sparse}: $2.0$. \Cref{fig:ogbench_all_tasks} reports per-task success rate across all 50 OGBench tasks for every method evaluated in~\cref{sec:exp:ogbench}.

\begin{table}[h]
\centering
\small
\setlength{\tabcolsep}{6pt}
\renewcommand{\arraystretch}{1.1}
\caption{Per-domain hyperparameters: action-chunking horizon $H$, discount $\gamma$, critic pessimism $\rho$, and tuned steering coefficient $\alpha$. All other hyperparameters (network sizes, optimizer, batch size, $K$, $N$, $\tau$, $\lambda_{\mathrm{diag}}$, $\lambda_{\mathrm{cons}}$) are shared across domains.}
\label{tab:domain_hparams}
\begin{tabular}{lcccc}
\toprule
Domain & $H$ & $\gamma$ & $\rho$ & $\alpha$ \\
\midrule
\texttt{antmaze-large-navigate}        & 1 & 0.99  & 0.5 & 0.2 \\
\texttt{antmaze-giant-navigate}        & 1 & 0.995 & 0.5 & 0.3 \\
\texttt{humanoidmaze-medium-navigate}  & 1 & 0.995 & 0.0 & 0.2 \\
\texttt{humanoidmaze-large-navigate}   & 1 & 0.995 & 0.0 & 0.2 \\
\texttt{cube-double-play}              & 5 & 0.99  & 0.5 & 0.3 \\
\texttt{cube-triple-play}              & 5 & 0.99  & 0.5 & 0.5 \\
\texttt{cube-quadruple-play}           & 5 & 0.99  & 0.5 & 0.5 \\
\texttt{scene-play-sparse}             & 5 & 0.99  & 0.5 & 0.2 \\
\texttt{puzzle-3x3-play-sparse}        & 5 & 0.99  & 0.5 & 0.3 \\
\texttt{puzzle-4x4-play-sparse}        & 5 & 0.99  & 0.5 & 2.0 \\
\bottomrule
\end{tabular}
\end{table}

\paragraph{LIBERO Per-Task Settings and Results.}
\label{app:implementation:hparams:libero}
\Cref{tab:libero_hparams} lists the six evaluation tasks, the natural-language goal description supplied to $\pi_{0.5}$, the zero-shot $\pi_{0.5}$-LIBERO success rate, and the per-task selected steering coefficient $\alpha$ (used by \qpilotu, \qpilotm, and CGQL-LinEx) and DSRL action magnitude $b_W$. \Cref{tab:libero_results} reports end-of-training success rate per task. The critic architecture is described in \cref{app:implementation:arch}; the LIBERO setting is pure online RL with no offline phase, BoN uses $N=5$, and reported success rates are EMA-smoothed (time-weighted, $\alpha=0.99$) along the training axis before averaging across seeds. The action-chunking horizon is $H=10$, matching $\pi_{0.5}$; the discount is $\gamma=0.99$ and critic pessimism is $\rho=0.5$ across all 6 tasks.

\begin{table}[h]
\centering
\small
\setlength{\tabcolsep}{5pt}
\renewcommand{\arraystretch}{1.1}
\caption{\textbf{LIBERO end-of-training success rate ($\%$).} Mean over $5$ seeds at $5\times10^5$ steps, EMA-smoothed (time-weighted, $\alpha=0.99$) along the training axis prior to averaging. Bold marks the best mean per column; differences within bootstrap CIs occur on tasks 26 and 60 between \qpilotu and DSRL.}
\label{tab:libero_results}
\begin{tabular}{lccccccc}
\toprule
Method & Task 26 & Task 31 & Task 38 & Task 59 & Task 60 & Task 64 & Mean \\
\midrule
$\pi_{0.5}$ (zero-shot) & 6 & 14 & 32 & 19 & 50 & 36 & 26 \\
\midrule
DSRL        & 36 & 22 & 70 & \textbf{93} & 95 & 34 & 58 \\
BoN ($N=5$) & 10 & 19 & 53 & 35 & 66 & 57 & 40 \\
EXPO        &  7 & 29 & 55 & 28 & 62 & 52 & 39 \\
CGQL-LinEx  & 11 & 12 & 25 & 27 & 64 & 49 & 31 \\
\midrule
\qpilotu      & \textbf{43} & \textbf{55} & \textbf{91} & 82 & \textbf{99} & \textbf{89} & \textbf{76} \\
\qpilotm      & 41 & 42 & 78 & 87 & 98 & 86 & 72 \\
\bottomrule
\end{tabular}
\end{table}

\begin{table}[h]
\centering
\small
\setlength{\tabcolsep}{4pt}
\renewcommand{\arraystretch}{1.1}
\caption{LIBERO-90 evaluation tasks. Zero-shot SR is the success rate of the frozen $\pi_{0.5}$-LIBERO checkpoint without any online learning. $\alpha$ is the steering coefficient used by \qpilotu, \qpilotm, and CGQL-LinEx; $b_W$ is DSRL's action magnitude.}
\label{tab:libero_hparams}
\begin{tabular}{cp{3.0in}ccc}
\toprule
Task & Description & Zero-shot SR (\%) & $\alpha$ & $b_W$ \\
\midrule
26 & put the wine bottle in the bottom drawer of the cabinet  & 5.9  & 0.3 & 2 \\
31 & put the black bowl on top of the cabinet                  & 13.8 & 0.3 & 3 \\
38 & put the right moka pot on the stove                       & 31.5 & 0.1 & 1 \\
59 & pick up the tomato sauce and put it in the tray           & 19.2 & 0.3 & 2 \\
60 & pick up the black bowl on the left and put it in the tray & 49.5 & 0.2 & 2 \\
64 & stack the right bowl on the left bowl and place them in the tray & 35.9 & 0.2 & 1 \\
\bottomrule
\end{tabular}
\end{table}

\paragraph{Compute.}
OGBench experiments use TPU v4-8 hosts. \texttt{cube-quadruple-play} and \texttt{puzzle-4x4-play-sparse} additionally use the OGBench 100M datasets.
A full $1.5 \times 10^6$-step run (offline + online) for one task and four seeds in parallel takes approximately 8 hours for \qpilotu and 12 hours for \qpilotm, depending on the task (different domains differ widely in evaluation time requirements). In online RL on LIBERO tasks, $5 \times 10^5$ takes about 6 hours for \qpilotu and 8 hours for \qpilotm.
Offline MFM distillation on LIBERO ($4 \times 10^4$ gradient steps, batch size $256$) takes approximately two days on a single NVIDIA H100 due to the cost of running the frozen $3$B-parameter $\pi_{0.5}$ teacher in every batch. This cost is paid once as we use the same MFM for all tasks.

\section{Additional Experiments}
\label{app:experiments}

\subsection{Latency Analysis}

We conduct an inference-time latency analysis per time step using $\pi_{0.5}$ model. Given that $\pi_{0.5}$ is a 3B-parameter model, it inherently incurs a relatively high inference cost. When employing the Best-of-$N$, the repeated inference cycles lead to a surge in latency, making real-time control challenging. As shown in~\cref{tab:inference_costs}, we perform single-step inference on task 59 across methods running on a single NVIDIA RTX PRO 6000. Using $N=20$ in Best-of-$N$ results in $136.9\%$ performance improvement with 1.44 times slower inference, but \qpilotu gets $+333.8\%$ with only $1.04$ times overhead, even faster than Best-of-5. Furthermore, while \qpilotm achieves the best performance on this task, it still struggles with the high latency due to the computational cost of multiple posterior estimations.

\begin{table}[h]
\centering
\small
\setlength{\tabcolsep}{8pt}
\renewcommand{\arraystretch}{1.2}
\caption{Comparison of success rate after training 500k update steps, inference time, and computational overhead. For the $\pi_{0.5}$ model using Best-of-$N$, the inference time escalates with $N$.}
\label{tab:inference_costs}
\begin{tabular}{lcccc}
\toprule
Method & $N$ & Success rate ($\%$) & Inference-time & Overhead \\
\midrule
$\pi_{0.5}$ & -  & 19.2 & 112.7ms & 1.0$\times$ \\
\multirow{3}{*}{Best-of-$N$} 
& 5  & 36.4 (+$89.5\%$) & 127.2ms & 1.14$\times$ \\
& 10 & 41.4 (+$115.6\%$) & 137.0ms & 1.22$\times$ \\
& 20 & 45.5 (+$136.9\%$) & 158.0ms & 1.44$\times$ \\
\qpilotu & - & 83.3 (+$333.8\%$) & \textbf{114.1ms} & \textbf{1.04$\times$} \\
\qpilotm & - & \textbf{88.5} (+$360.9\%$) & 240.1ms & 2.14$\times$ \\
\bottomrule
\end{tabular}
\end{table}

\begin{figure}[t]
\centering
\makebox[\textwidth][c]{\includegraphics[width=0.7\paperwidth-1in\relax,keepaspectratio]{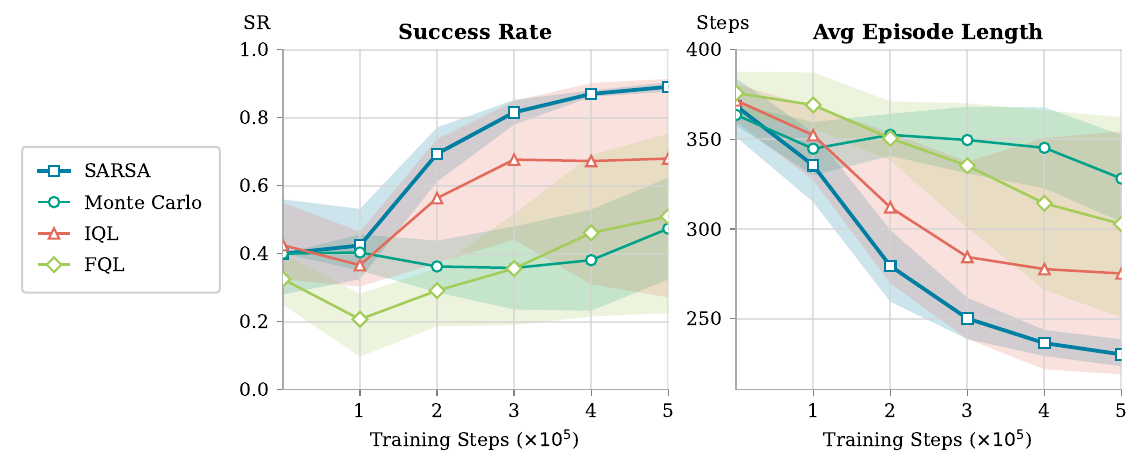}}
\caption{\textbf{Value Learning Method Ablation.} We compare SARSA with IQL, FQL, and Monte-Carlo on the LIBERO-90 task 64. It shows means over 5 seeds and shaded bands show seed-level $95\%$ bootstrap CIs. A lower average episode length indicates that the method can complete the task faster. }
\label{fig:value_ablation}
\end{figure}

\subsection{Value Learning Ablation}
\label{app:value_ablation}
VLA's slow inference makes standard Bellman target-based TD-learning almost infeasible, as each update requires querying the policy "batch\_size" times. To find an alternative, we compare 4 value-learning objectives for the critic on LIBERO-90 task 64: SARSA (we choose for \qpilots), Monte Carlo (MC) returns, Implicit Q-Learning (IQL)~\citep{kostrikov2022iql}, and Flow Q-Learning (FQL)~\citep{park2025fql}. Note that FQL is not used for policy learning. Instead, we still use the steered $\pi_{0.5}$ for inference, but distill the steered $\pi_{0.5}$ policy into a one-step policy for Q-learning next action query. 

As shown in~\cref{fig:value_ablation}, SARSA reaches a success rate of $90\%$ at $5\times10^5$ training steps, outperforming the other 3 approaches. Despite being unbiased, the MC estimator regresses directly against discounted episode returns, which are dominated by large negative rewards in early training, leading to slower initial learning across seeds. IQL and FQL are originally designed for offline RL. However, under the non-stationary replay buffer of our online setting, both methods exhibit instability (high variance), suggesting that current offline value learning approaches still struggle with the continually shifting action distribution. In contrast, SARSA is shown as a pragmatic and effective choice for online critic training under a generalist policy.

\subsection{Sensitivity Analysis}
\label{app:alpha_sweep}

\cref{fig:alpha_sweep} reports the per-domain steering-coefficient sweep used to select $\alpha$ in \cref{app:implementation:hparams}. For each OGBench domain, we run \qpilotu on the two tuning tasks (tasks 1 and 4 for navigation, tasks 2 and 4 for manipulation) over $\alpha \in \{0.1, 0.2, 0.3, 0.5, 1.0, 2.0\}$ with 4 seeds per configuration, and pick the $\alpha$ with the strongest combined performance for that domain. Most domains peak between $0.1$ and $0.5$ and degrade gracefully on either side, indicating low sensitivity inside this range. The exception is \texttt{puzzle-4x4-play-sparse}, where success is essentially flat below $1.0$ and only emerges at $\alpha = 2.0$; we therefore extend the sweep to $2.0$ on that domain only.

\begin{figure}[p]
\vspace*{-0.5in}
\centering
\makebox[\textwidth][c]{\includegraphics[width=\dimexpr\paperwidth-1in\relax,keepaspectratio]{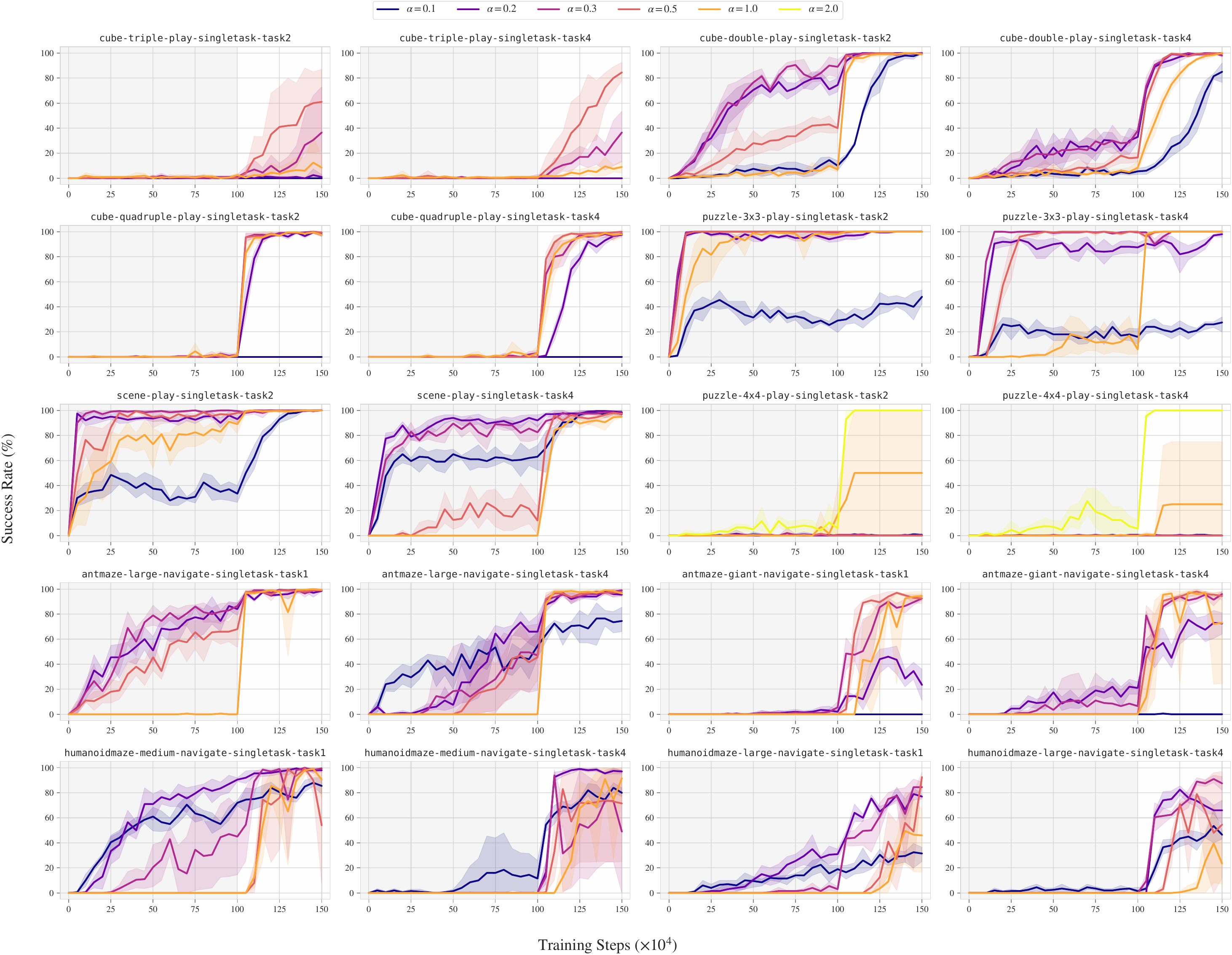}}
\caption{\textbf{Steering Coefficient $\alpha$ Sweep.} Eval success on the two tuning tasks of each OGBench domain across $\alpha \in \{0.1, 0.2, 0.3, 0.5, 1.0, 2.0\}$. Curves show seed-bootstrap mean and $95\%$ CI. We only sweep $2.0$ on \texttt{puzzle-4x4} since performance on others peaked well before $1.0$.}
\label{fig:alpha_sweep}
\end{figure}
\clearpage

\begin{figure}[p]
\vspace*{-0.5in}
\centering
\makebox[\textwidth][c]{\includegraphics[width=\dimexpr\paperwidth-1in\relax,keepaspectratio]{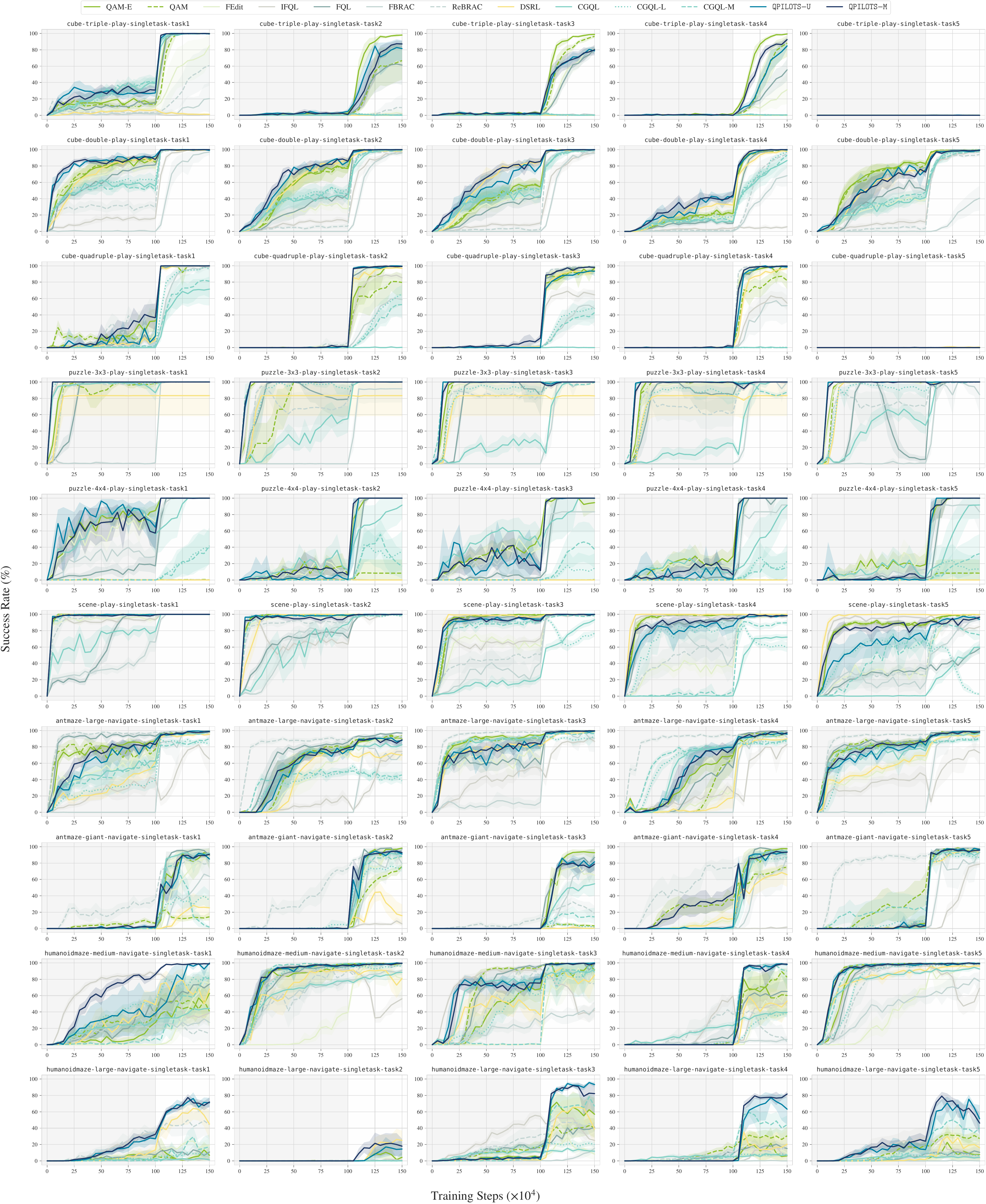}}
\caption{\textbf{OGBench: Per-Task Success Rate Across All 50 Tasks.}}
\label{fig:ogbench_all_tasks}
\end{figure}

\begin{figure}[h]
\centering
\includegraphics[width=5.5in]{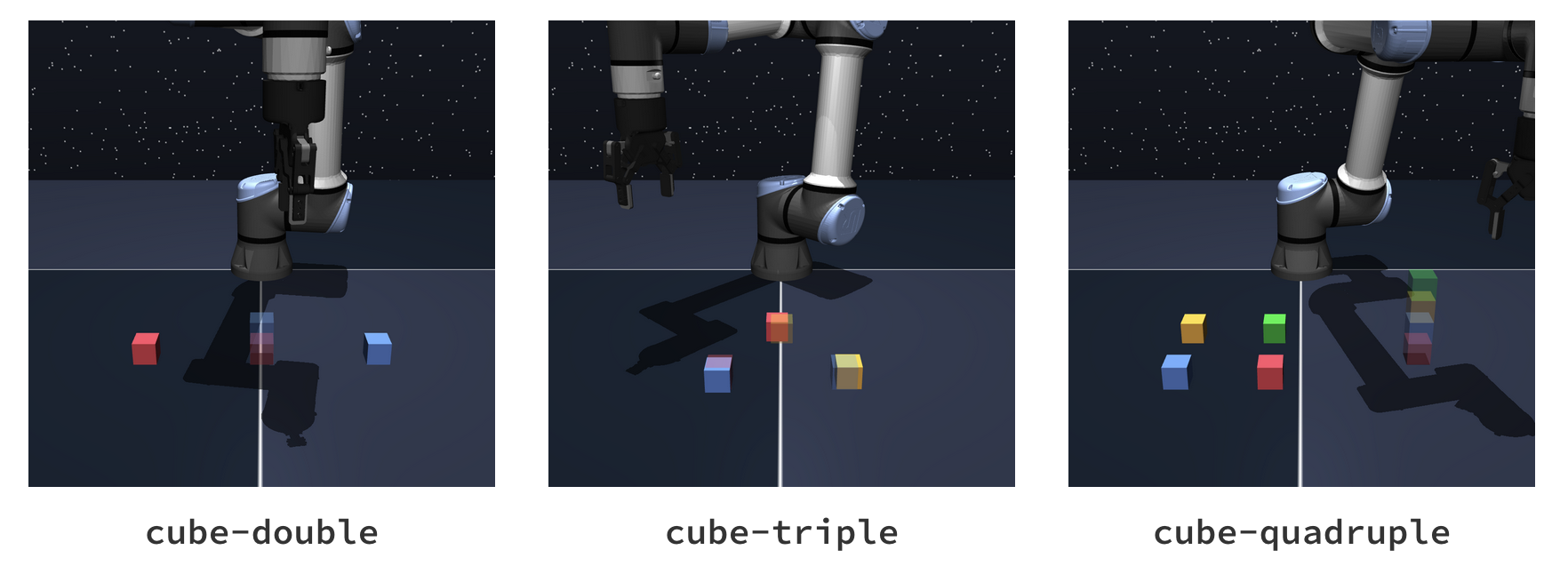}
\includegraphics[width=5.5in]{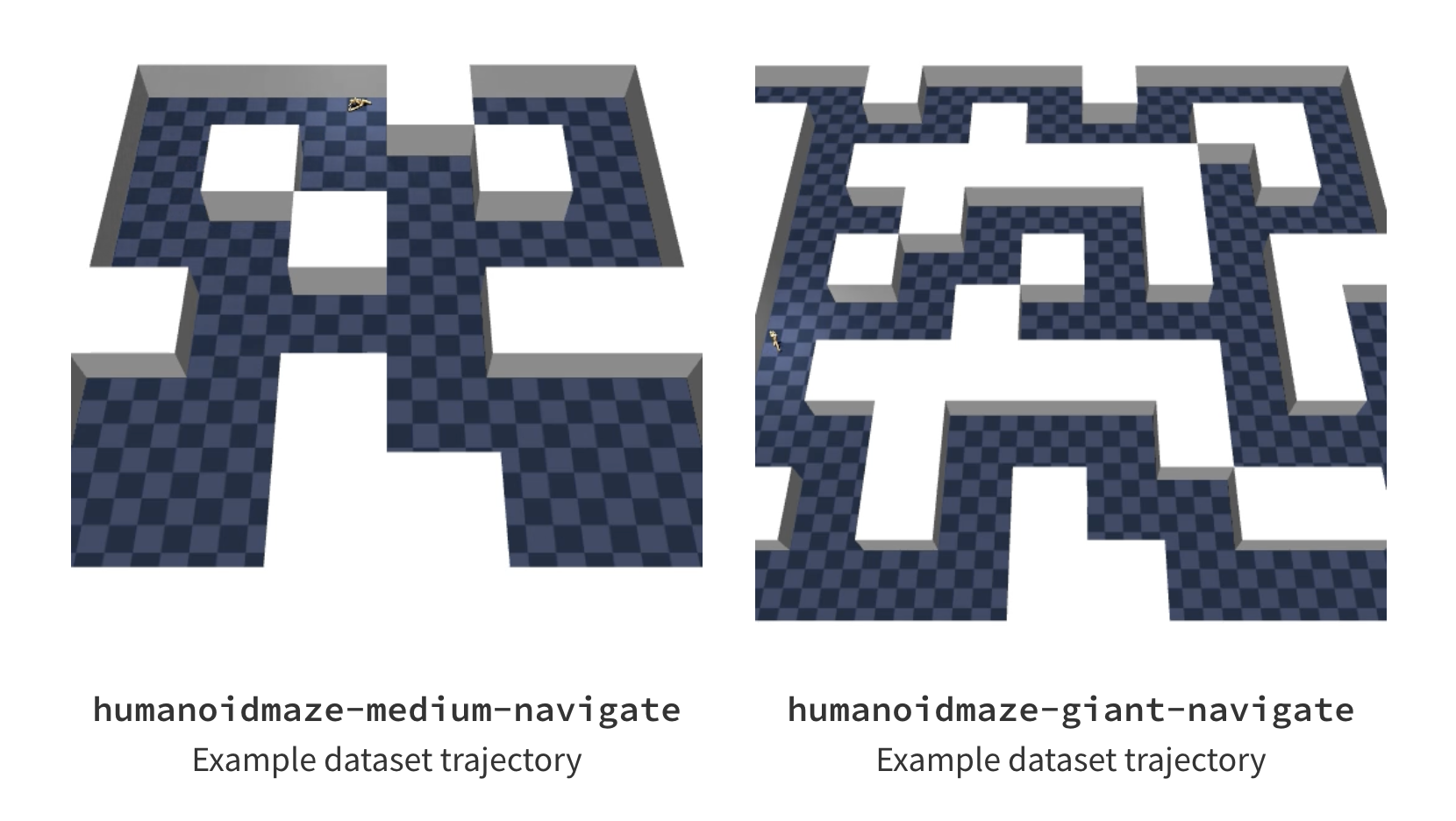}
\caption{\textbf{OGBench Domains.}
\textbf{Top:} \texttt{cube} tasks involve controlling a robot arm to achieve 5 different goal configurations per domain.
\textbf{Bottom:} \texttt{humanoid} tasks involve controlling a 21-DoF humanoid to navigate to different goal locations, and is a particularly challenging long-horizon domain within the task-suite. \cite{park2025ogbench}.}
\label{fig:banner}
\end{figure}

\begin{figure}[h]
\begin{center}
\includegraphics[width=2.24in]{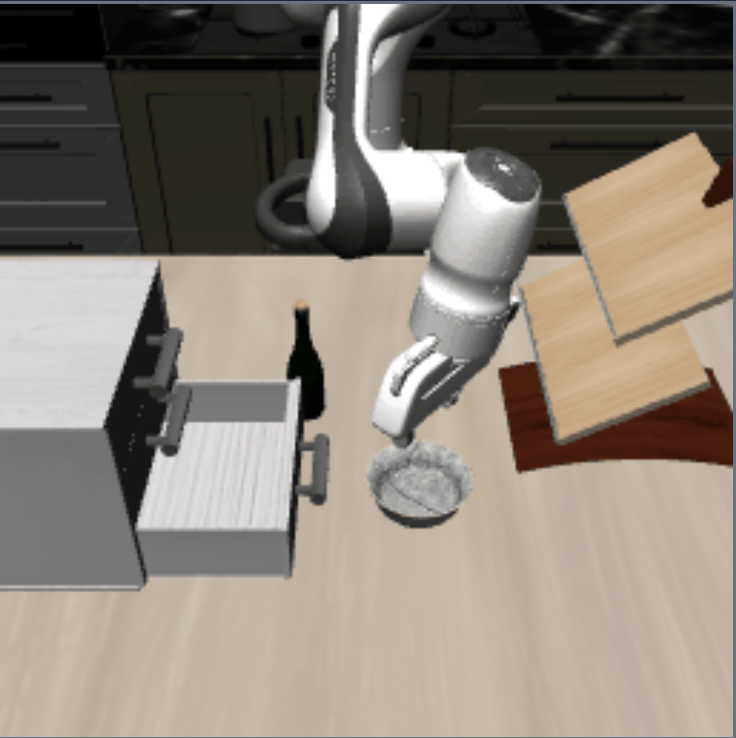}
\includegraphics[width=2.24in]{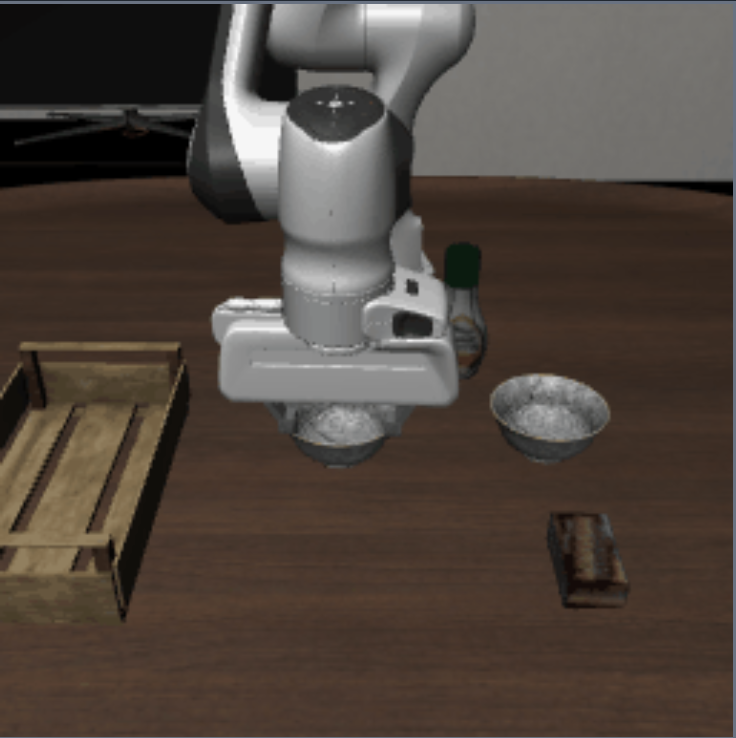}
\includegraphics[width=2.24in]{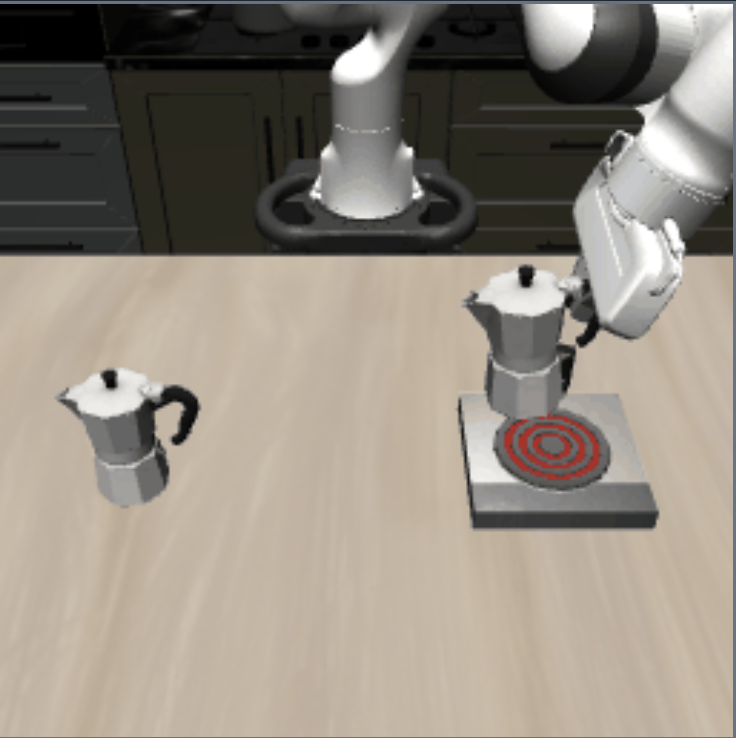}
\includegraphics[width=2.24in]{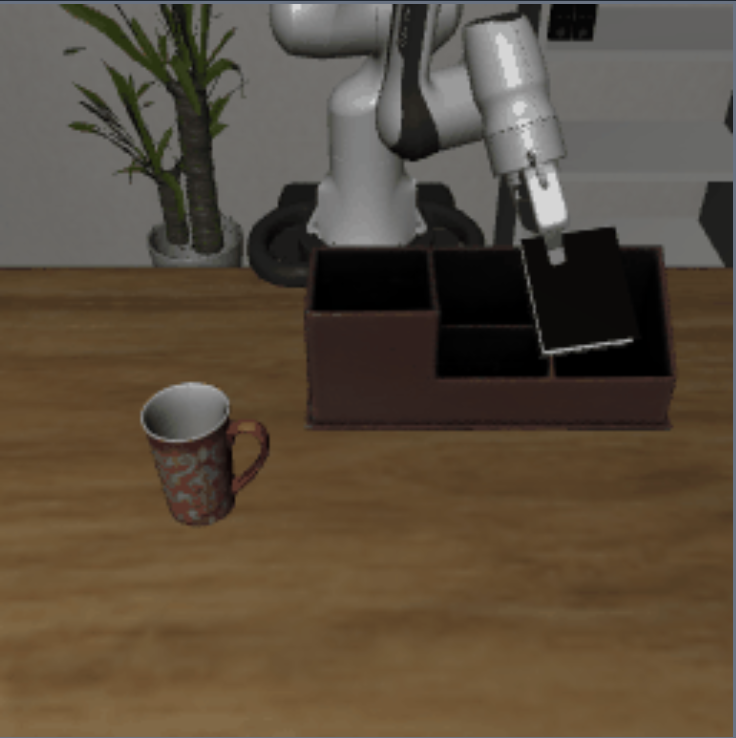} 
\end{center}
\caption{\textbf{LIBERO Tasks.} The LIBERO simulation task suite involves controlling a 7-DoF Franka arm to manipulate various household objects. Tasks vary in difficulty and length.
\cite{liu2023libero}.}
\label{fig:banner}
\end{figure}

\clearpage

\newpage

\end{document}